\documentclass{article}[12pt]

\usepackage[utf8]{inputenc} 
\usepackage[T1]{fontenc}    
\usepackage{hyperref}       
\usepackage{url}            
\usepackage{booktabs}       
\usepackage{amsfonts}       
\usepackage{nicefrac}       
\usepackage{microtype}      
\usepackage{titlesec}

\usepackage{amsmath,amssymb,amsthm}
\usepackage{fullpage}
\usepackage{enumitem}
\usepackage{graphics,graphicx,color}
\usepackage{multirow}
\usepackage{mathrsfs}
\usepackage{subfig}
\usepackage{float}
\usepackage{hyperref}
\usepackage{tabulary}
\usepackage{thm-restate}
\usepackage{cleveref}

\newcommand{\Lap}{\text{Lap}}

\newcommand{\E}[1]{\mathbb{E}\left[#1\right]}

\newcommand{\eps}{\varepsilon}
\renewcommand{\epsilon}{\varepsilon}
\newcommand{\poly}{\mathrm{poly}}

\usepackage{xcolor}
\definecolor{DarkGreen}{rgb}{0.1,0.5,0.1}

\newcommand{\new}[1]{#1}

\newcommand{\SAFFRON}{{\textsc{SAFFRON}}}
\newcommand{\PSAFFRON}{{\textsc{PAPRIKA}}}
\newcommand{\AboveThresh}{{\textsc{AboveThresh}}}
\newcommand{\ReportMin}{{\textsc{ReportNoisyMin}}}
\newcommand{\SparseVector}{{\textsc{SparseVector}}}

\usepackage{algorithmicx,algpseudocode, algorithm}

\newtheorem{theorem}{Theorem}
\newtheorem{lemma}{Lemma}

\newtheorem{definition}{Definition}

\def\icmlver{0}

\begin{document}

\title{\PSAFFRON: Private Online False Discovery Rate Control}

\ifnum\icmlver=0
\author{Wanrong Zhang\footnotemark[1] \and Gautam Kamath\footnotemark[2] \footnotemark[3] \and Rachel Cummings\footnotemark[1] \footnotemark[3]}

\renewcommand{\thefootnote}{\fnsymbol{footnote}}
\footnotetext[1]{H. Milton Stewart School of Industrial and Systems Engineering, Georgia Institute of Technology. {\tt \{wanrongz, rachelc\}@gatech.edu}. W.Z. supported in part by a Mozilla Research Grant, NSF grant CNS-1850187, and an ARC-TRIAD Fellowship from the Georgia Institute of Technology. R.C. supported in part by a Mozilla Research Grant, a Google Research Fellowship, and NSF grant CNS-1850187.  This work was initiated while both authors were visiting the Simons Institute for the Theory of Computing.}
\footnotetext[2]{Cheriton School of Computer Science, University of Waterloo. {\tt g@csail.mit.edu}. Supported by a University of Waterloo startup grant. This work was initiated while the author was a Microsoft Research Fellow at the Simons Institute for the Theory of Computing at UC Berkeley.}
\footnotetext[3]{Indicates equal contribution as last author.}

\renewcommand{\thefootnote}{\arabic{footnote}}

\fi
\maketitle

\begin{abstract}

In hypothesis testing, a \emph{false discovery} occurs when a hypothesis is incorrectly rejected due to noise in the sample. When adaptively testing multiple hypotheses, the probability of a false discovery increases as more tests are performed. Thus the problem of \emph{False Discovery Rate (FDR) control} is to find a procedure for testing multiple hypotheses that accounts for this effect in determining the set of hypotheses to reject. The goal is to minimize the number (or fraction) of false discoveries, while maintaining a high true positive rate (i.e., correct discoveries).

In this work, we study False Discovery Rate (FDR) control in multiple hypothesis testing under the constraint of differential privacy for the sample. 
Unlike previous work in this direction, we focus on the \emph{online setting}, meaning that a decision about each hypothesis must be made immediately after the test is performed, rather than waiting for the output of all tests as in the offline setting.
We provide new private algorithms based on state-of-the-art results in non-private online FDR control.
Our algorithms have strong provable guarantees for privacy and statistical performance as measured by FDR and power.
We also provide experimental results to demonstrate the efficacy of our algorithms in a variety of data environments.

\end{abstract}


\section{Introduction}

In the modern era of big data, data analyses play an important role in decision-making in healthcare, information technology, and government agencies. The growing availability of large-scale datasets and ease of data analysis, while beneficial to society, has created a severe crisis of reproducibility in science. 
In 2011, Bayer HealthCare reviewed 67 in-house projects and found that they could replicate fewer than 25 percent, and found that over two-thirds of the projects had major inconsistencies \cite{engineering2019reproducibility}. 
One major reason is that random noise in the data can often be mistaken for interesting signals, which does not lead to valid and reproducible results. 
This problem is particularly relevant when testing multiple hypotheses, when there is an increased chance of false discoveries based on noise in the data. 
For example, an analyst may conduct 250 hypothesis tests and find that 11 are significant at the 5\% level. 
This may be exciting to the researcher who publishes a paper based on these findings, but elementary statistics suggests that (in expectation) 12.5 of those tests should be significant at that level purely by chance, even if the null hypotheses were all true. 
To avoid such problems, statisticians have developed tools for controlling overall error rates when performing multiple hypothesis tests. 

In hypothesis testing, the \emph{null hypothesis} of no interesting scientific discovery (e.g., a drug has no effect), is tested against the alternative hypothesis of a particular scientific theory being true (e.g., a drug has a particular effect).  The significance of each test is measured by a \emph{$p$-value}, which is the probability of the observed data occurring under the null hypothesis, and a hypothesis is \emph{rejected} if the corresponding $p$-value is below some (fixed) significance level. Each rejection is called a \emph{discovery}, and a rejected hypothesis is a \emph{false discovery} if the null hypothesis is actually true.  When testing multiple hypotheses, the probability of a false discovery increases as more tests are performed.  The problem of \emph{false discovery rate (FDR) control} is to find a procedure for testing multiple hypotheses that takes in the $p$-values of each test, and outputs a set of hypotheses to reject. The goal is to minimize the number of false discoveries, while maintaining high true positive rate (i.e., true discoveries).

In many applications, the dataset may contain sensitive personal information, and the analysis must be conducted in a privacy-preserving way.  For example, in genome-wide association studies (GWAS), a large number of single-nucleotide polymorphisms (SNPs) are tested for an association with a disease simultaneously or adaptively. Prior work has shown that the statistical analysis of these datasets can lead to privacy concerns, and it is possible to identify an individual's genotype when only minor allele frequencies are revealed \cite{homer2008resolving}. The field of {\em differential privacy} \cite{DMNS06} offers data analysis tools that provide powerful worst-case privacy guarantees, and has become a de facto gold standard in private data analysis. Informally, an algorithm that is $\eps$-differentially private ensures that any particular output of the algorithm is at most $e^\eps$ more likely when a single data point is changed. This parameterization allows for a smooth tradeoff between accurate analysis and privacy to the individuals who have contributed data. In the past decade, researchers have developed a wide variety of differentially private algorithms for many statistical tasks; these tools have been implemented in practice at major organizations including Google~\cite{ErlingssonPK14}, Apple~\cite{AppleDP17}, Microsoft~\cite{DingKY17}, and the U.S. Census Bureau~\cite{DajaniLSKRMGDGKKLSSVA17}.

\paragraph{Related Work.}
\setlength{\textfloatsep}{5pt}
The only prior work on differentially private FDR control \cite{dwork2018differentially} considers the classic offline multiple testing problem, where an analyst has all the hypotheses and corresponding $p$-values upfront. Their private method repeatedly applies \ReportMin\ \cite{dwork2014algorithmic} to the celebrated Benjamini-Hochberg (BH) procedure \cite{benjamini1995controlling} in offline multiple testing to privately pre-screen the $p$-values, and then applies the BH procedure again to select the significant $p$-values.  The (non-private) BH procedure first sorts all $p$-values, and then sequentially compares them to an increasing threshold, where all $p$-values below their (ranked and sequential) threshold are rejected.  The \ReportMin\ mechanism privatizes this procedure by repeatedly (and privately) finding the hypothesis with the lowest $p$-value.

Although the work of \cite{dwork2018differentially} showed that it was possible to integrate differential privacy with FDR control in multiple hypothesis testing, the assumption of having all hypotheses and $p$-values upfront is not reasonable in many practical settings.  For example, a hospital may conduct multi-phase clinical trials where more patients join over time, or a marketing company may perform A/B testings sequentially. In this work, we focus on the more practical \emph{online hypothesis testing problem}, where a stream of hypotheses arrive sequentially, and decisions to reject hypotheses must be made based on current and previous results before the next hypothesis arrives. This sequence of the hypotheses could be independent or adaptively chosen. Due to the fundamental difference between the offline and online FDR procedures, the method of \cite{dwork2018differentially} \new{based on  \ReportMin} cannot be applied to the online setting. \new{Instead, we use \SparseVector, described in Section \ref{sec.dpprelim}, as a starting point.} 
Discussion of non-private online multiple hypothesis testing appears in Section~\ref{s.background}.

\nocite{TianR19}



\paragraph{Our Results.}
We develop a differentially private online FDR control procedure for multiple hypothesis testing, which takes a stream of $p$-values and a target FDR level and privacy parameter $\varepsilon$, and outputs discoveries that can control the FDR at a certain level at any time point.  Such a procedure provides unconditional differential privacy guarantees (to ensure that privacy will be protected even in the worst case) and satisfy the theoretical guarantees dictated by the FDR control problem.

Our algorithm, Private Alpha-investing P-value Rejecting Iterative sparse veKtor Algorithm (\PSAFFRON, Algorithm~\ref{algo.psaffron}), is presented in Section~\ref{s.algo}.
Its privacy and accuracy guarantees are stated in Theorem~\ref{thm.priv} and~\ref{thm.fdr}, respectively.
While the full proofs appear in the appendix, we describe the main ideas behind the algorithms and proofs in the surrounding prose.
In Section~\ref{s.exper}, we provide a thorough empirical investigation of \PSAFFRON, with additional empirical results in Appendix \ref{app.plots}.


\section{Preliminaries}

\subsection{Background on Differential Privacy}
Differential Privacy bounds the maximal amount that one data entry can change the output of the computation. Databases belong to the space $\mathcal{D}^n$ and contain $n$ entries--one for each individual--where each entry belongs to data universe $\mathcal{D}$. We say that $D,D' \in \mathcal{D}^n$ are \emph{neighboring databases} if they differ in at most one data entry.

\begin{definition}[Differential Privacy \cite{DMNS06}]\label{def.dp}
	An algorithm $\mathcal{M}: \mathcal{D}^n \rightarrow \mathcal{R}$ is \emph{$(\epsilon,\delta)$-differentially private} if for every pair of neighboring databases $D,D' \in \mathbb{R}^n$, and for every subset of possible outputs $\mathcal{S} \subseteq \mathcal{R}$,
	$\Pr[\mathcal{M}(D) \in \mathcal{S}] \leq \exp(\epsilon)\Pr[\mathcal{M}(D') \in \mathcal{S}] + \delta$. 
	If $\delta = 0$, we say that $\mathcal{M}$ is {\em $\epsilon$-differentially private}.
\end{definition}

The \emph{additive sensitivity} of a real-valued query $f:\mathcal{D}^n \to \mathbb{R}$ is denoted $\Delta f$, and is defined to be the maximum change in the function's value that can be caused by changing a single entry.  That is,
\[ \Delta f = \max_{D,D' \text{ neighbors }} \left| f(D) - f(D') \right|. \]
If $f$ is a vector-valued query, the expression above can be modified with the appropriate norm in place of the absolute value. Differential privacy guarantees are often achieved by adding \emph{Laplace noise} at various places in the computation, where the noise scales with $\Delta f/\epsilon$.  A Laplace random variable with parameter $b$ is denoted $\Lap(b)$, and has probability density function,
\[ p_{\Lap(b)}(x) = \frac{1}{2b} \exp\left(\frac{-|x|}{b}\right) \quad \forall x \in \mathbb{R}. \]
We may sometimes abuse notation and also use $\Lap(b)$ to denote the realization of a random variable with this distribution.

The \SparseVector\ algorithm, first introduced by \cite{DNPR10} and refined to its current form by \cite{dwork2014algorithmic}, privately reports the outcomes of a potentially very large number of computations, provided that only a few are ``significant.'' It takes in a stream of queries, and releases a bit vector indicating whether or not each noisy query answer is above the fixed noisy threshold. We use this algorithm as a framework for our online private false discovery rate control algorithm as new hypotheses arrive online, and we only care about those ``significant'' hypotheses when the $p$-value is below a certain threshold.  We note that the standard presentation below checks for queries with values above a threshold, but by simply changing signs this framework can be used to check for values \emph{below} a threshold, as we will do with the $p$-values.

{\centering
	\begin{minipage}{\linewidth}
		\begin{algorithm}[H]
			\caption{Sparse Vector: \SparseVector($D, \Delta, \{f_1, f_2, \ldots \}, T, c, \epsilon$) }
			\begin{algorithmic}
				\State \textbf{Input:} database $D$, stream of queries $\{f_1, f_2, \ldots \}$ each with sensitivity $\Delta$, threshold $T$, a cutoff point $c$, privacy parameter $\epsilon$
				\State Let $\hat{T}_0 = T + \Lap(\frac{2\Delta c}{\epsilon})$
				\State Let count $=0$
				\For {each query $i$}
				\State Let $Z_i \sim \Lap(\frac{4\Delta c}{\epsilon})$
				\If {$f_i(X)+Z_i>\hat{T}$}
				\State Output $a_i=\top$
				\State Let count $=$ count $+1$
				\State Let $\hat{T}_{\text{count}}=T+\Lap(\frac{2\Delta c}{\epsilon})$
				\Else
				\State Output $a_i=\bot$
				\EndIf
				\If{count $\ge c$}
				\State Halt.
				\EndIf
				\EndFor
			\end{algorithmic}
		\end{algorithm}
	\end{minipage}
	
	\begin{theorem}[\cite{DNPR10}]
		\SparseVector\ is $(\epsilon, 0)$-differentially private. 
	\end{theorem}
}

\begin{theorem}[\cite{DNPR10}]\label{thm.atacc}
	For any sequence of $k$ queries $f_1,\ldots, f_k$ with sensitivity $\Delta$ such that $|\{i: f_i(D)\ge T-\alpha\}|\le c$, \SparseVector\ outputs with probability at least $1-\beta$ a stream of $a_1, \ldots, a_k \in \{\top, \bot\}$ such that $a_i=\bot$ for every $i\in[m]$ with $f(i)<T-\alpha_{SV}$ and $a_i=\top$ for every $i\in[m]$ with $f(i)>T+\alpha_{SV}$ as long as
	$\alpha_{SV}\ge \frac{8\Delta c\log (2k c/\beta)}{\epsilon}$.
\end{theorem}

Unlike the conventional use of additive sensitivity, \cite{dwork2018differentially} defined the notion of {\em multiplicative sensitivity} specifically for $p$-values. It is motivated by the observation that, although the additive sensitivity of a $p$-value may be large, the relative change of the $p$-value on two neighboring datasets is stable unless the $p$-value is very small. Using this alternative sensitivity notion means that preserving privacy for these $p$-values only requires a small amount of noise.

\begin{definition}[Multiplicative Sensitivity \cite{dwork2018differentially}]\label{def.multsens}
	A p-value function $p$ is said to be $(\eta,\mu)$-multiplicative sensitive if for all neighboring databases $D$ and $D'$, either both $p(D), p(D')\le \mu$ or 
	\[ exp(-\eta)p(D)\le p(D') \le \exp(\eta)p(D). \]
\end{definition}

Specifically, when $\mu$ is sufficiently small, then we can treat the logarithm of the $p$-values as having additive sensitivity $\eta$, and we only need to add noise that scales with $\eta/\epsilon$, which may be much smaller than the noise required under the standard additive sensitivity notion.

\subsection{Background on Online False Discovery Rate Control}
\label{s.background}
In the online false discovery rate (FDR) control problem, a data analyst receives a stream of hypotheses on the database $D$, or equivalently, a stream of $p$-values $p_1,p_2,\ldots$. The analyst must pick a threshold $\alpha_t$ at each time $t$ to reject the hypothesis when $p_t\le \alpha_t$; this threshold can depend on previous hypotheses and discoveries, and rejection must be decided before the next hypothesis arrives. 


The error metric is the false discovery rate, formally defined as: 
\[ \text{FDR}=\E{\text{FDP}}=\E{\frac{|\mathcal{H}^0\cap\mathcal{R}|}{|\mathcal{R}|}}, \]
where $\mathcal{H}^0$ is the (unknown to the analyst) set of hypotheses where the null hypothesis is true, and $\mathcal{R}$ is the set of rejected hypotheses.  We will also write these terms as a function of time $t$ to indicate their values after the first $t$ hypotheses: $\text{FDR}(t),\; \text{FDP}(t),\; \mathcal{H}^0(t), \; \mathcal{R}(t)$.
The goal of FDR control is to guarantee that for any time $t$, the FDR up to time $t$ is less than a pre-determined quantity $\alpha\in(0,1)$. 

Such a problem was first investigated by \cite{foster2008alpha}, who proposed a framework known as \emph{online alpha-investing} that models the hypothesis testing problem as an investment problem.  The analyst is endowed with an initial budget, can test hypotheses at a unit cost, and receives an additional reward for each discovery. The alpha-investing procedure ensures that the analysts always maintains an $\alpha$-fraction of their wealth, and can therefore continue testing future hypotheses indefinitely.  Unfortunately, this approach only controls a slightly relaxed version of FDR, known as \emph{mFDR}, which is given by
$\text{mFDR}(t)=\frac{\E{|\mathcal{H}^0\cap\mathcal{R}|}}{\E{|\mathcal{R}|}}$.
This approach was later extended to a class of generalized alpha-investing (GAI) rules \cite{aharoni2014generalized}. One subclass of GAI rules, the Level based On Recent Discovery (LORD), was shown to have consistently good performance in practice \cite{javanmard2015online,javanmard2018online}. The SAFFRON procedure, proposed by \cite{ramdas2018saffron}, further improves the LORD procedures by adaptively estimating the proportion of true nulls. The SAFFRON procedure is the current state-of-the-art in online FDR control for multiple hypothesis testing.

To understand the main differences between the SAFFRON and the LORD procedures, we first introduce an oracle estimate of the FDP as
$\text{FDP}^*(t)=\frac{\sum_{j\le t, j\in \mathcal{H}^0}\alpha_j}{|\mathcal{R}(t)|}$. The numerator $\sum_{j\le t, j\in \mathcal{H}^0}\alpha_j$ overestimates the number of false discoveries, so $\text{FDP}^*(t)$ overestimates the $\text{FDP}$. The oracle estimator $\text{FDP}^*(t)$ cannot be calculated since $\mathcal{H}^0$ is unknown. LORD's naive estimator $\sum_{j\le t}\alpha_j/|\mathcal{R}(t)|$ is a natural overestimate of $\text{FDP}^*(t)$.
The SAFFRON's threshold sequence is based on a novel estimate of FDP as
\begin{equation}\label{eq.safffdp}
\widehat{\text{FDP}}_{\text{SAFFRON}}(t)=\frac{\sum_{j\le t}\alpha_j\frac{I(p_j>\lambda_j)}{1-\lambda_j}}{|\mathcal{R}(t)|},
\end{equation}
where $\{\lambda_j\}_{j=1}^\infty$ is a sequence of user-chosen parameters in the interval $(0,1)$, which can be a constant or a deterministic function of the information up to time $t-1$. This is a much better estimator than LORD's naive estimator $\sum_{j\le t}\alpha_j/|\mathcal{R}(t)|$. The SAFFRON estimator is a fairly tight estimate of $\text{FDP}^*(t)$, since intuitively $I(p_j>\lambda_j)/(1-\lambda_j)$ has unit expectation under null hypotheses and is stochastically smaller than uniform under non-null hypotheses.

The SAFFRON algorithm is given formally in Algorithm \ref{algo.saffron}. SAFFRON starts off with an error budget $(1-\lambda_1)W_0<(1-\lambda_1)\alpha$, which will be allocated to different tests over time. It never loses wealth when testing candidate $p$-values with $p_j<\lambda_j$, and it earns back wealth of $(1-\lambda_j)\alpha$ on every rejection except for the first.  By construction, the SAFFRON algorithm controls $\widehat{\text{FDP}}_{\text{SAFFRON}}(t)$ to be less than $\alpha$ at any time $t$. The function $g_t$ for defining the sequence $\{\lambda_j\}_{j=1}^\infty$ can be any coordinatewise non-decreasing function. For example, $\{\lambda_j\}_{j=1}^\infty$ can be a deterministic sequence of constants, or $\lambda_t=\alpha_t$, as in the case of alpha-investing.  These $\lambda_j$ values serve as a weak overestimate of $\alpha_j$. The algorithm first checks if a $p$-value is below $\lambda_j$, and if so, adds it to the \emph{candidate set} of hypotheses that may be rejected.  It then computes the $\alpha_j$ threshold based on current wealth, current size of the candidate set, and the number of rejections so far, and decides to reject the hypothesis if $p_j \leq \alpha_j$.  It also takes in a non-increasing sequence of decay factors $\gamma_j$ which sum to one.  These decay factors serve to depreciate past wealth and ensure that the sum of the wealth budget is always below the desired level $\alpha$.

{\centering
	\begin{minipage}{\linewidth}
		\begin{algorithm}[H]
			\caption{\SAFFRON($\alpha, W_0, \{\gamma_j\}_{j=0}^\infty$)}
			\begin{algorithmic}
				\State \textbf{Input:} stream of $p$-values $\{p_1, p_2, \ldots\}$, target FDR level $\alpha$, initial wealth $W_0<\alpha$, positive non-increasing sequence $\{\gamma_j\}_{j=0}^\infty$ of summing to one.
				\State Set rejection number $i=0$
				\For{each $p$-value $p_t$}
				\State Set $\lambda_t=g_t(R_{1:t-1},C_{1:t-1})$
				\State Set the indicator for candidacy $C_t=I( p_t<\lambda_t)$. Set the candidates after the $j$-th rejection as $C_{j+}=\sum_{i=\tau_j+1}^{t-1}C_i$
				\If{$t=1$}
				\State Set $\alpha_1=(1-\lambda_1)\gamma_1W_0$
				\Else
				\State Compute $\alpha_t=(1-\lambda_t)(W_0\gamma_{t-C_{0+}}+(\alpha-W_0)\gamma_{t-\tau_1-C_{1+}}+\sum_{j\ge 2}\alpha\gamma_{t-\tau_j-C_{j+}})$
				\EndIf
				\State Output $R_t=I(p_t\le \alpha_t)$
				\If{$R_t=1$}
				\State Update rejection number $i=i+1$. Set the $i$-th rejection time as $\tau_i=t$
				\EndIf				
				\EndFor
			\end{algorithmic}\label{algo.saffron}
		\end{algorithm}
	\end{minipage}
}

The \SAFFRON\ algorithm requires that the input sequence of $p$-values are not too correlated under the null hypothesis.  This condition is formalized through a \emph{filtration} on the sequence of candidacy and rejection decisions.  Intuitively, this means that the sequence of hypotheses cannot be too adaptively chosen, otherwise the $p$-values may become overly correlated and violate this condition. Denote by $R_j:=I(p_j\le\alpha_j)$ the indicator for rejection, and let $C_j:=I(p_j\le \lambda_j)$ be the indicator for candidacy. Define the filtration formed by the sequences of $\sigma$-fields $\mathcal{F}^t:=\sigma(R_1,\ldots,R_t,C_1,\ldots,C_t)$, and let $\alpha_t:=f_t(R_1,\ldots,R_{t-1},C_1,\ldots,C_{t-1})$, where $f_t$ is an arbitrary function of the first $t-1$ indicators for rejections and candidacy. We say that the null $p$-values are conditionally super-uniformly distributed with respect to the filtration $\mathcal{F}$ if:

\begin{equation}\label{assump.saffron}
\text{If null hypothesis } H_i \text{ is true, then } \Pr(p_t\le \alpha_t|\mathcal{F}^{t-1})\le \alpha_t.
\end{equation}

We note that independent $p$-values is a special case of the conditional super-uniformity condition of \eqref{assump.saffron}.  When $p$-values are independent, they satisfy the following condition:



\begin{equation*}
\text{If the null hypothesis } H_i \text{ is true, then } \Pr(p_t\le u)\le u \text{ for all } u\in[0,1].
\end{equation*}

\noindent SAFFRON provides the following accuracy guarantees, where the first two conditions apply if $p$-values are conditionally super-uniformly distributed, and the last two conditions apply if the $p$-values are additionally independent under the null.

\begin{theorem}[\cite{ramdas2018saffron}]\label{thm.saffron} If the null $p$-values are conditionally super-uniformly distributed, then we have:\\
	(a) $\E{\sum_{j\le t, j\in \mathcal{H}^0}\alpha_j\frac{I(p_j>\lambda_j)}{1-\lambda_j}}\ge \E{|\mathcal{H}^0\cap \mathcal{R}(t)|}$;\\
	(b) The condition $\widehat{\text{FDP}}_{\text{SAFFRON}}(t) \le \alpha$ for all $t\in \mathbb{N}$ implies that mFDR$(t)\le \alpha$ for all $t\in \mathbb{N}$.\\
	\noindent If the null $p$-values are independent of each other and of the non-null $p$-values, and $\{\alpha_t\}$ and $\{\lambda_t\}$ are coordinatewise non-decreasing functions of the vector $R_1,\ldots, R_{t-1},C_1,\ldots, C_{t-1}$, then \\
	(c) $\E{\widehat{\text{FDP}}_{\text{SAFFRON}}(t)}\ge \E{FDP(t)} := FDR(t)$ for all $t\in \mathbb{N}$;\\
	(d) The condition $\widehat{\text{FDP}}_{\text{SAFFRON}}(t) \le \alpha$ for all $t$ implies that $FDR(t)\le \alpha$ for all $t\in \mathbb{N}$.
\end{theorem}


\section{Private online false discovery rate control}\label{s.algo}



In this section, we provide our algorithm for private online false discovery rate control, \PSAFFRON, given formally in Algorithm \ref{algo.psaffron}. It is a differentially private version of \SAFFRON, where we use \SparseVector\ to ensure privacy of our rejection set.  However, the combination of these tools is far from immediate, \new{and several algorithmic innovations are required, including: dynamic thresholds in \SparseVector\ to match the alpha-investing rule of SAFFRON, adding noise that scales with the multiplicative sensitivity of $p$-values to reduce the noise required for privacy, shifting the SparseVector threshold to accommodate FDR as a novel accuracy metric, and the candidacy indicator step which cannot be done privately and requires new analysis for both privacy and accuracy.} Complete proofs of our privacy and accuracy results appear in the appendix; we elaborate here on the algorithmic details and why these modifications are needed to ensure privacy and FDR control.

{\centering
	\begin{minipage}{\linewidth}
		\begin{algorithm}[H]
			\caption{ \PSAFFRON($\alpha, \lambda, W_0, \gamma, c, \epsilon, \delta,s$)}
			\begin{algorithmic}
				\State \textbf{Input:} stream of $p$-values $\{p_1, p_2, \ldots \}$ with mutiplicative sensitivity ($\eta$,$\mu$), target FDR level $\alpha$, initial wealth $W_0<\alpha$, positive non-increasing sequence $\{\gamma_j\}_{j=0}^\infty$ of summing to one, expected number of rejections $c$, privacy parameters $\eps,\delta$, threshold shift magnitude $s$, maximum number of $p$-values $k$.
				\State Let $Z_\alpha^0\sim \Lap(2\eta c/\eps)$, count $=0$,\\ $A=\frac{sc\eta}{\eps} \log\frac{2}{3\min\{\delta,1-((1-\delta)/\exp(\eps))^{1/k}\}}$
				\For{each p-value $p_t$}
				\State \textbf{if} count $\ge c$ \textbf{then} Output $R_t=0$
				\State \textbf{else} 
				\State \indent Sample $Z_t\sim \Lap(4\eta c/\eps)$. Set $\lambda_t=g_t(R_{1:t-1},C_{1:t-1})$. Set the indicator for candidacy $C_t=I(\log p_t<\log2 \lambda_t)$. 
				\State \indent \textbf{if} $t=1$ 
				\State \indent \indent \textbf{then} Set $\alpha_1=(1-2\lambda_1)\gamma_1W_0$
				\State \indent \textbf{else} 
				\State \indent  \indent  Compute $\alpha_t=(1-2\lambda_t)(W_0\gamma_{t}+(\alpha-W_0)\gamma_{t-\tau_1}+\sum_{j\ge 2}\alpha\gamma_{t-\tau_j})$
				\State \indent \textbf{if} $C_t=1$ and $\log p_t+Z_t\le \log \alpha_t-A+Z_\alpha^{\text{count}}$ 
				\State \indent \indent \textbf{then} Output $R_t=1$. Set count = count $+1$ and sample $Z_\alpha^{\text{count}}\sim \Lap(2\eta c/\eps)$
				\State \indent \textbf{else} Output $R_t=0$
				\EndFor
			\end{algorithmic}\label{algo.psaffron}
		\end{algorithm}
	\end{minipage}
}

The \SAFFRON~algorithm decides to reject hypothesis $t$ if the corresponding $p$-value $p_t$ is less than the rejection threshold $\alpha_t$; that is, if $p_t \leq \alpha_t$.  We instantiate the \SparseVector\ framework in this setting, where $p_t$ plays the role of the $t^{th}$ query answer $f_t(X)$, and $\alpha_t$ plays the role of the threshold. Note that \SparseVector\ uses a single fixed threshold for all queries, while our algorithm \PSAFFRON\ allows for a dynamic threshold that depends on the previous output.  Our privacy analysis of the algorithm accounts for this change and shows that dynamic thresholds do not affect the privacy guarantees of \SparseVector. However, the algorithm would not be private if the dynamic thresholds also depend on the data. Note that \SAFFRON\ never loses wealth when testing candidate $p$-values with $p_j\le \lambda_j$, and the threshold $\alpha_j$ depends on the data since it is based on current wealth. We remove such dependence in \PSAFFRON\ by losing wealth at every step regardless of whether we test a candidate $p$-values, similar to LORD. This will result in stricter FDR control (and potentially weaker power) because our wealth decays faster. 

Similar to prior work on private offline FDR control \cite{dwork2018differentially}, we use {\em multiplicative sensitivity} as described in Definition \ref{def.multsens}, as $p$-values may have high sensitivity and require unacceptably large noise to be added to preserve privacy.  We assume that our input stream of $p$-values $p_1,p_2,\ldots,$ each has multiplicative sensitivity $(\eta,\mu)$. As long as $\mu$ is small enough (i.e., less than the rejection threshold), we can treat the logarithm of the $p$-values as the queries with additive sensitivity $\eta$.  Because of this change, we must make rejection decisions based on the logarithm of the $p$-values, so our reject condition is $\log p_t + Z_t \leq \log \alpha_t + Z_{\alpha}$ for Laplace noise terms $Z_t, Z_{\alpha}$ drawn from the appropriate distributions.

The accuracy guarantees of \SparseVector\ ensure that if a value is reported to be below threshold, then with high probability it will not be more than $\alpha_{SV}$ above the threshold.  However, to ensure that our algorithm satisfies the desired bound $FDR\leq \alpha$, we require that reports of ``below threshold'' truly do correspond to $p$-values that are below the desired threshold $\alpha_t$.  To accommodate this, we shift our rejection threshold $\log \alpha_t$ down by a parameter $A$.
$A$ is chosen such that the algorithm satisfies $(\varepsilon, \delta)$-differential privacy, but the choice can be seen as inspired by the $\alpha_{SV}$-accuracy term of \SparseVector\ as given in Theorem \ref{thm.atacc}. 
Therefore our final reject condition is $\log p_t + Z_t \leq \log \alpha_t - A + Z_{\alpha}$. This ensures that ``below threshold'' reports are below $(\log \alpha_t - A) + \alpha_{SV} \approx \log \alpha_t$ with high probability.  Empirically, we see that the bound of $A$ in Theorem \ref{thm.priv} may be overly conservative and lead to no hypotheses being rejected, so we allow an additional scaling parameter $s$ that will scale the magnitude of shift by a factor of $s$.  The conservative bounds of Theorem \ref{thm.priv} correspond to $s=4$, but in many scenarios a smaller value of $s=1$ or $2$ will lead to better performance while still satisfying the privacy guarantee.  Further guidance choosing this shift parameter is given in Section \ref{sec:shift}. 

Even with these modifications, a naive combination of \SparseVector\ and \SAFFRON\ would still not satisfy differential privacy. This is due to the \emph{candidacy indicator} step of the algorithm.  In the \SAFFRON\ algorithm, a pre-processing candidacy step occurs before any rejection decisions.  This step checks whether each $p$-value $p_t$ is smaller than a loose upper bound $\lambda_t$ on the eventual reject threshold $\alpha_t$.  The algorithm chooses $\alpha_t$ using an $\alpha$-investing rule that depends on the number of candidate hypotheses seen so far, and ensures that $\alpha_t \leq \lambda_t$, so only hypotheses in this candidate set can be rejected. These $\lambda$ values are used to control $\widehat{\text{FDP}}_{\text{SAFFRON}}(t)$, which serves as a conservative overestimate of FDP$(t)$. (For a discussion of how to choose $\lambda_t$, see Lemma \ref{lemma1} or our experimental results in Section \ref{s.exper}. Reasonable choices would be $\lambda_t=\alpha_t$ or a small constant such as $0.2$.)

Without adding noise to the candidacy condition, there may be neighboring databases with $p$-values $p_t, p_t'$ for some hypothesis such that $\log p_t < \log \lambda_t < \log p_t'$, and hence the hypothesis would have positive probability of being rejected under the first database and zero probability of rejection under the neighbor.  This would violate the $(\epsilon,0)$-differential privacy guarantee intended under \SparseVector. If we were to privatize the condition for candidacy using, for example, a parallel instantiation of \SparseVector, then we would have to reuse the same realizations of the noise when computing the rejection threshold $\alpha_t$ to still control FDP, but this would no longer be private. 

Since we cannot add noise to the candidacy condition, we weaken it in \PSAFFRON\ to be $\log p_t < \log 2\lambda_t$. Then if a hypothesis has different candidacy results under neighboring databases and the multiplicative sensitivity $\eta$ is small, then the hypothesis is still extremely unlikely to be rejected even under the database for which it was candidate. To see this, consider a pair of neighboring databases that induce $p$-values where $\log p_t < \log 2\lambda_t < \log p_t'$.  Due to the multiplicative sensitivity constraint, we know that $\log p_t \geq \log 2\lambda_t -\eta$. Plugging this into the rejection condition $\log p_t + Z_t \leq \log \alpha_t - A + Z_{\alpha}$, we see that we would need the difference of the noise terms to satisfy $Z_t - Z_{\alpha} \leq \log \frac12 -A + \eta$, which by analysis of the Laplace distribution, will happen with exponentially small probability in $n$ when $\eta=\poly^{-1}(n)$.\footnote{Such values of $\eta$ are typical; see examples in Section \ref{s.exper} where $\eta=\frac{1}{\sqrt{n}}$. The shift term $A$ also has dependence on $\eta$ which contributes to the bound.}  Our \PSAFFRON\ algorithm is thus $(\epsilon, \delta)$-differentially private, and we account for this failure probability in our (exponentially small) $\delta$ parameter, as stated in Theorem \ref{thm.priv}. 

Our \PSAFFRON\ algorithm allows analysts to specify a maximum number of hypotheses tested $k$ and rejections $c$. We require a bound on the maximum number of hypotheses tested because the accuracy guarantees of \SparseVector\ only allows exponentially (in the size of the database) many queries to be answered accurately. Once the total number of rejections reaches $c$, the algorithm will fail to reject all future hypotheses. We do not halt the algorithm as in \SparseVector\, and therefore, \PSAFFRON\  does not have a stopping criterion, and we can safely talk about the FDR control at any fixed time, just like \SAFFRON.

Our algorithm also controls at each time $t$,
$\widehat{\text{FDP}}_{\PSAFFRON}(t)\le \frac{\sum_{j\le t}\alpha_t\frac{I(p_j>2\lambda_j)}{1-2\lambda_j}}{|\mathcal{R}(t)|}$.
We note that this is equivalent to $\widehat{\text{FDP}}_{\text{SAFFRON}}(t)$ by scaling down $\lambda_j$ by a factor of 2.  By analyzing and bounding this expression, we achieve FDR bounds for our \PSAFFRON\ algorithm, as stated in Theorem \ref{thm.fdr}.

\begin{restatable}{theorem}{privthm}\label{thm.priv} 
For any stream of p-values $\{p_1, p_2, \ldots\}$, \PSAFFRON\ is $(\eps,\delta)$-differentially private. 
\end{restatable}

As a starting point, our privacy comes from \SparseVector, but as discussed above, many crucial modifications are required.
To briefly summarize the key considerations, we must handle different thresholds at different times, multiplicative rather than additive sensitivity, a modified notion of the candidate set, and introducing a small delta parameter to account for the new candidate set definition and the shift.
The proof of Theorem~\ref{thm.priv} appears in Appendix~\ref{privacy.proofs}.

Next we describe the theoretical guarantees of FDR control for our private algorithm \PSAFFRON\, which is an analog of Theorem \ref{thm.saffron}. 
We modify \new{the notation of} the conditional super-uniformity assumption of \SAFFRON\ to incorporate the added Laplace noise. \new{The conditions are otherwise identical. (See \eqref{assump.saffron} in Appendix \ref{app.algs} for comparison.) We note that independent $p$-values is a special case of conditional super-uniformity, but this requirement more generally allows for a broader class of dependencies among $p$-values.}
Let $R_j:=I(p_j+Z_j\le\alpha_j+Z_\alpha)$ be the rejection decisions, and let $C_j:=I(p_j\le 2\lambda_j)$ be the indicators for candidacy. 
We let $\alpha_t:=f_t(R_1,\ldots,R_{t-1},C_1,\ldots,C_{t-1})$, where $f_t$ is an arbitrary function of the first $t-1$ indicators for rejections and candidacy. Define the filtration formed by the sequences of $\sigma$-fields $\mathcal{F'}^t:=\sigma(R_1,\ldots,R_t,C_1,\ldots,C_t, Z_1,\ldots,Z_t,Z_\alpha)$. The null $p$-values are conditionally super-uniformly distributed with respect to the filtration $\mathcal{F'}$ if when the null hypothesis $H_i$ is true, then $\Pr(p_t\le \alpha_t|\mathcal{F'}^{t-1})\le \alpha_t$.
\new{We emphasize that this condition is only needed for FDR control, and that our privacy guarantee of Theorem~\ref{thm.priv} holds for arbitrary streams of $p$-values, even those which do not satisfy conditional super-uniformity.}

Our FDR control guarantees for \PSAFFRON\ mirror those of \SAFFRON\ (Theorem \ref{thm.saffron}).  The first two statements apply if $p$-values are conditionally super-uniform, and the last two statements apply if the $p$-values are additionally independent under the null. The proof of Theorem~\ref{thm.fdr} appears in Appendix~\ref{sec.fdr-appendix}.

\begin{restatable}{theorem}{fdrthm}\label{thm.fdr} 
If the null $p$-values are conditionally super-uniformly distributed, then we have:\\
	(a) $\E{\sum_{j\le t, j\in \mathcal{H}^0}\alpha_j\frac{I(p_j>2\lambda_j)}{1-2\lambda_j}}\ge \E{|\mathcal{H}^0\cap \mathcal{R}(t)|}$;\\
  (b)The condition $\widehat{\text{FDP}}_{\PSAFFRON}(t)\le \alpha$ for all $t \in \mathbb{N}$ implies that mFDR$(t)\le \alpha+\delta t$ for all $t \in  \mathbb{N}$.\\
\noindent If the null $p$-values are independent of each other and of the non-null $p$-values, and $\{\alpha_t\}$ and $\{\lambda_t\}$ are coordinate-wise non-decreasing functions of the vector $R_1,\ldots, R_{t-1},C_1,\ldots, C_{t-1}$, then \\
  (c) $\E{\widehat{\text{FDP}}_{\PSAFFRON}(t)}\ge \E{FDP(t)} := FDR(t)$ for all $t \in \mathbb{N}$;\\
  (d) The condition $\widehat{\text{FDP}}_{\PSAFFRON}(t)\le \alpha$ for all $t$ implies that $FDR(t)\le \alpha+\delta t$ for all $t \in \mathbb{N}$.
\end{restatable}

Relative to the non-private guarantees of Theorem~\ref{thm.saffron}, the FDR bounds provided by \PSAFFRON\ are weaker by an additive of $\delta t$. \new{In most differential privacy applications, $\delta$ is typically required to be cryptographically small (i.e., at most negligible in the size of the database) \cite{dwork2014algorithmic}, so this additional term should have a minuscule effect on the FDR.\footnote{Alternatively, $\delta$ could be treated like a tunable parameter to balance the tradeoff between privacy and FDR control. If an analyst has an upper bound on the allowable slack in FDR, say 0.01, then she could set $\delta=0.01/t$ to ensure her desired bound.} We note that $\epsilon$ plays a role in the analysis of Theorem \ref{thm.fdr}, although it does not appear in FDR bounds.  See \eqref{eq.lap2} in the appendix, where the term with dependence on $\epsilon$ can be upper bounded by $\delta$ for any $\eps>0$.}

The following lemma is a key tool in the proof of Theorem~\ref{thm.fdr}.
Though it is qualitatively similar to Lemma 2 in~\cite{ramdas2018saffron}, it is crucially modified to show an analogous statement holds under the addition of Laplace noise.
Its proof appears in Appendix~\ref{sec.lemma-proof}.

\begin{restatable}{lemma}{lemfirst}\label{lemma1}
	Assume $p_1,p_2,\ldots$ are all independent and let $h:\{0,1\}^k\rightarrow R$ be any coordinate-wise non-decreasing function. Assume $f_t$ and $g_t$ are coordinate-wise non-decreasing functions and that $\alpha_t=f_t(R_{1:t-1},C_{1:t-1})$ and $\lambda_t=g_t(R_{1:t-1},C_{1:t-1})$. Then for any $t\le k$ such that $H_t\in \mathcal{H}^0$, we have $\E{\frac{\alpha_tI(p_t>2\lambda_t)}{(1-2\lambda_t)h(R_{1:k})}|\mathcal{F'}^{t-1}} \ge \E{\frac{\alpha_t}{h(R_{1:k})}|\mathcal{F'}^{t-1}}$ and $ \E{\frac{\min\{\alpha_t\exp(Z_\alpha-Z_t-A),1\}}{h(R_{1:k})}|\mathcal{F'}^{t-1}} \ge \E{\frac{I(\log p_t +Z_t\le \log \alpha_t+Z_\alpha-A)}{h(R_{1:k})}|\mathcal{F'}^{t-1}}$.
\end{restatable}


\new{There are no known theoretical bounds on the statistical power of \SAFFRON\, even in the non-private setting.  Instead, we validate power empirically through the experimental results in Section \ref{s.exper}.}

\section{Experiments}
\label{s.exper}
We experimentally compare the FDR and the statistical power of variations of the \PSAFFRON\ and \SAFFRON\ procedures, under different sequences of $\{\lambda_j\}$. Following the convention of \cite{ramdas2018saffron}, we define \PSAFFRON-Alpha-Investing, or \PSAFFRON\ AI, to be the instantiation of Algorithm \ref{algo.psaffron} with the sequence $\lambda_j=\alpha_j$, where the rejection threshold matches the $\alpha$-investing rule, and we use \PSAFFRON\ to denote Algorithm \ref{algo.psaffron} instantiated with a sequence of constant of $\lambda_j$, which in our experiments is $\lambda_j=0.2$. We use $\lambda_j=0.5$ in \SAFFRON.\footnote{Recall from Section 3 that our $\lambda_j$ is equivalent to the $\lambda_j$ in \SAFFRON\ scaling down by a factor of 2.} We generally observe that, even under moderately stringent privacy restrictions, \PSAFFRON\ and its AI variant perform comparably to the non-private alternatives, \new{with \PSAFFRON\ AI typically outperforming \PSAFFRON. This suggests that even though setting $\lambda_j$ as a fixed constant may be easier for implementation, parameter optimization can lead to meaningful performance improvements.} We chose the sequence $\{\gamma_j\}$ to be a constant $1/k$ up to time $k$. Note that the sequence can be decreasing such as of the form $\gamma_j \propto j^{-s}$ in \cite{ramdas2018saffron}, which controls the wealth to be more concentrated around small values of $j$. See \cite{ramdas2018saffron} for more discussion on the choice of $\{\gamma_j\}$. In our experiments, we set  the target FDR level $\alpha+\delta t=0.2$, and thus our privacy parameter $\delta$ is set to be bounded by $0.2/800 = 2.5 \times 10^{-4}$. The maximum number of rejections $c=40$. All the results are averaged over $100$ runs.
We investigate two settings: in Section~\ref{sec:bernoulli}, the observations come Bernoulli distributions, and in Section~\ref{sec:expon}, the observations are generated from truncated exponential distributions.
In Section~\ref{sec:shift}, we discuss our choice of the shift parameter $A$ and give guidance on how to choose this parameter in practice.
Code for \PSAFFRON\ and our experiments is available at \url{https://github.com/wanrongz/PAPRIKA}.

\vspace{-0.2cm}

\subsection{Testing with Bernoulli Observations}\label{sec:bernoulli}
We assume that the database $D$ contains $n$ individuals with $k$ independent features.
The $i$th feature is associated with $n$ i.i.d.\ Bernoulli variables $\xi_1^i,\ldots,\xi_n^i$, each of which takes the value $1$ with probability $\theta_i$, and takes the value $0$ otherwise. Let $t_i$ be the sum of the $i$th features. A $p$-value for testing null hypothesis $H_0^i: \theta_i\le 1/2$ against $H_1^i: \theta_i>1/2$ is given by 
\begin{equation*}
p_i(D)=\sum_{k=t_i}^n\frac{1}{2^n}{n \choose k}.
\end{equation*}
\cite{dwork2018differentially} showed that $p_i$ is $(\mu, \eta)$-multiplicatively sensitive for $\mu=m^{-1-c}$ and $\eta\asymp\sqrt{\frac{\log n}{n}}$, where $m\le \text{poly}(n)$ and $c$ is any small positive constant.

We choose $\theta_i$ for our experiments as follows:
\begin{equation*}
\theta_i=\begin{cases}
0.5 \quad \text{with probability } 1-\pi_1\\
0.75 \quad \text{with probability } \pi_1,
\end{cases}
\end{equation*}
for varying values of $\pi_1$, \new{which represents the expected fraction of non-null hypotheses.  We consider relatively small values of $\pi_1$ as most practical applications of FDR control (such as GWAS studies) will have only a small fraction of true ``discoveries'' in the data.}

In the following experiments, we sequentially test $H_0^i$ versus $H_1^i$ for $i=1, \ldots,k$. We use $n=1000$ as the size of the database $D$, and $k=800$ as the number of features as well as the number of hypotheses. 
Our experiments are run under several different shifts $A$, but due to space constraints, we only report results in the main body with $A=\frac{c\eta}{\eps} \log\frac{2}{3\min\{\delta,1-((1-\delta)/\exp(\eps))^{1/k}\}}$ (i.e., when $s=1$), which still satisfies our privacy guarantee.
Further discussion on the choice of $A$ and additional results under other shift parameters $s$ are deferred to Appendix \ref{sec:shift}. The results are summarized in Figure \ref{fig:example_bernoulli}, which plots the FDR and statistical power against the expected fraction of non-nulls, $\pi_1$. In Figure \ref{fig:example_bernoulli}(a) and (b), we compare our algorithms with privacy parameter $\eps=5$ to the non-private baseline methods of LORD \cite{javanmard2015online, javanmard2018online}, Alpha-investing \cite{aharoni2014generalized}, and \SAFFRON\ and \SAFFRON\ AI from \cite{ramdas2018saffron}. In Figure \ref{fig:example_bernoulli}(c,d) and (e,f), we compare the performance of \PSAFFRON\ AI and \PSAFFRON, respectively, with varying privacy parameters $\eps=3,5,10$. We also list these values in Table \ref{table.bin} in Appendix \ref{app.plots}.



\begin{center}
	\begin{minipage}{0.8\linewidth}
		\begin{figure}[H]
			\centering
			\subfloat[][$\epsilon=5$]{\includegraphics[width=.45\textwidth]{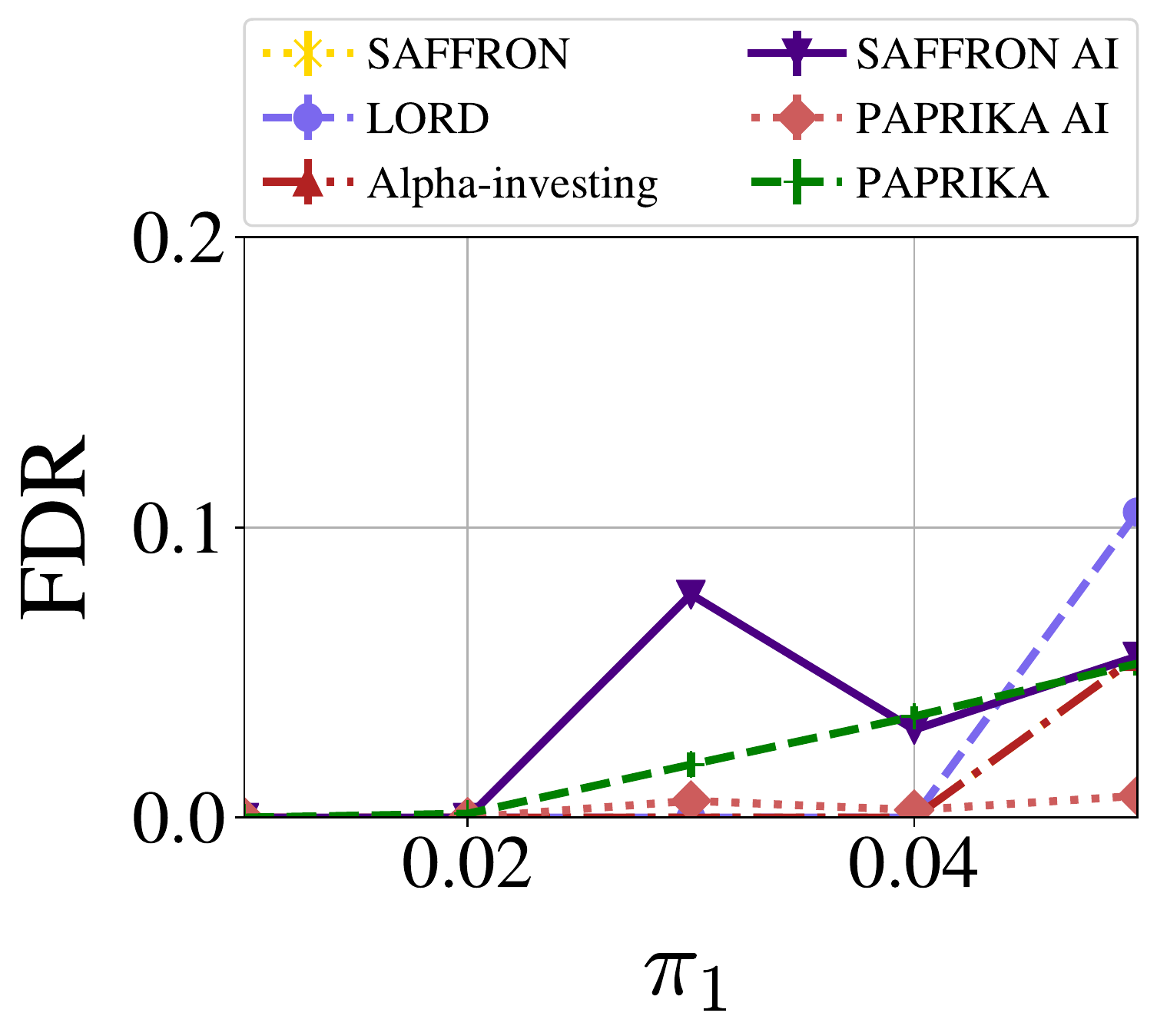}} 
			\subfloat[][$\epsilon=5$]{\includegraphics[width=.45\textwidth]{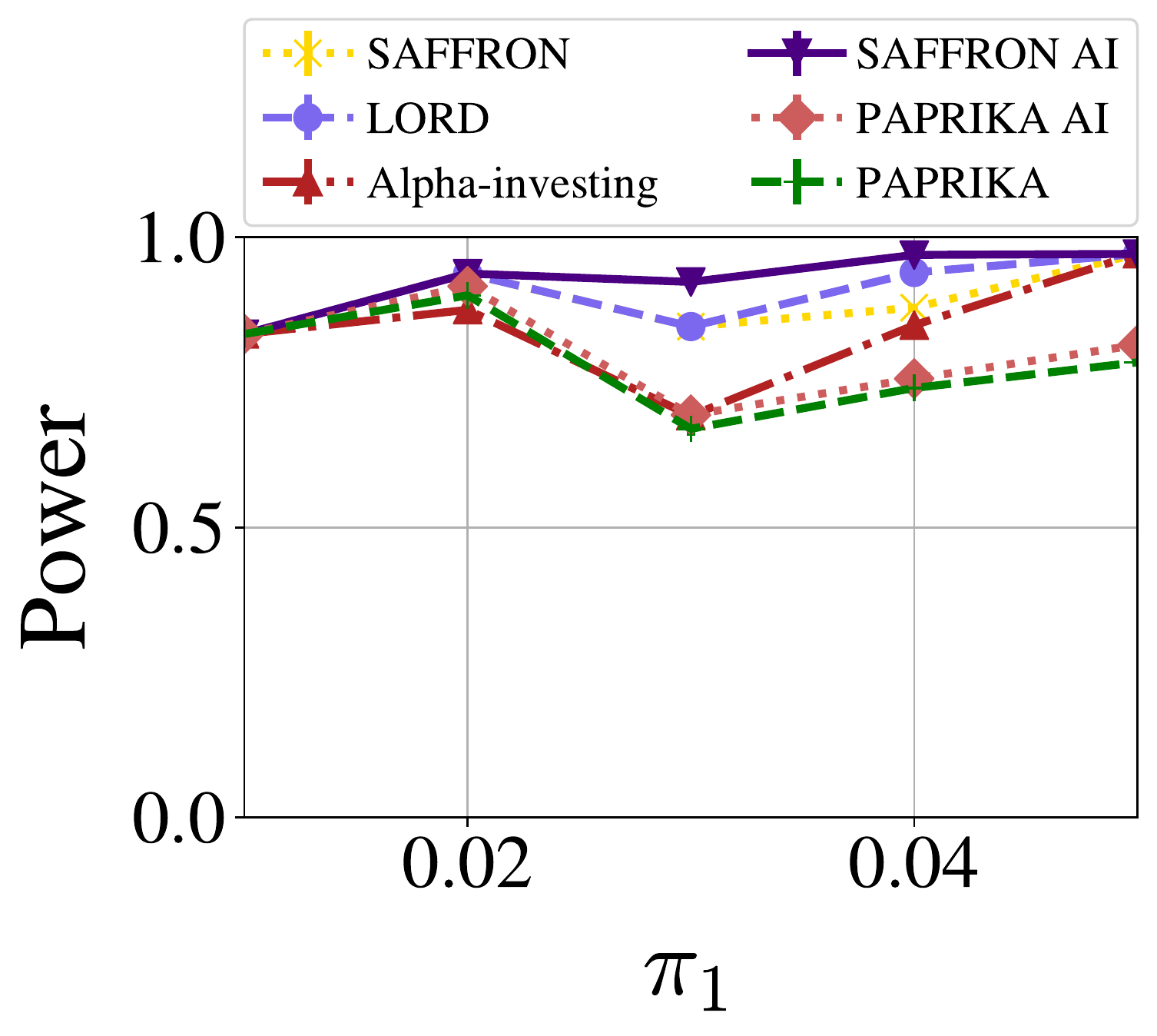}}\\\vspace{-0.2cm}
			\subfloat[][PAPRIKA AI]{\includegraphics[width=.45\textwidth]{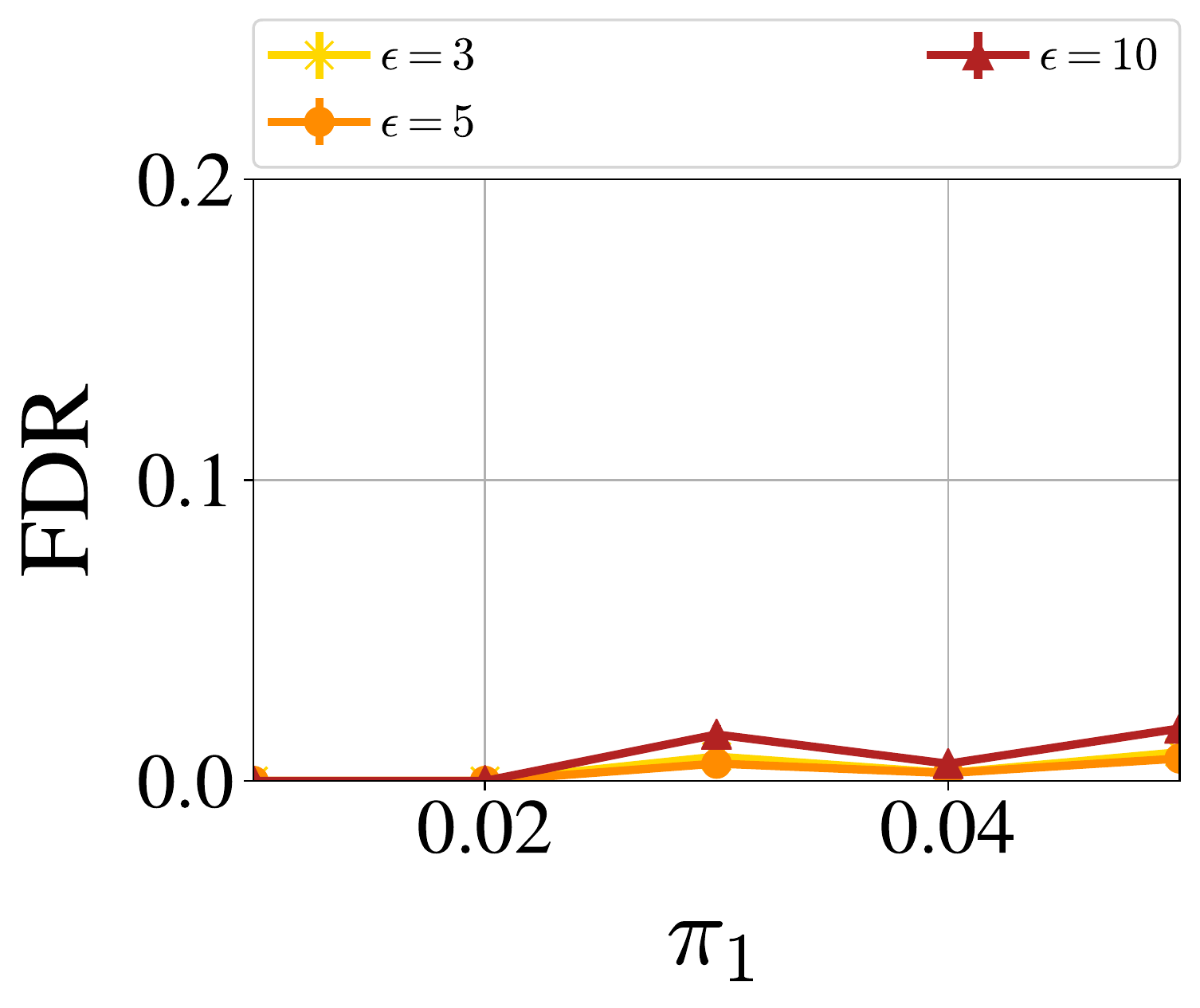}}
			\subfloat[][PAPRIKA AI]{\includegraphics[width=.45\textwidth]{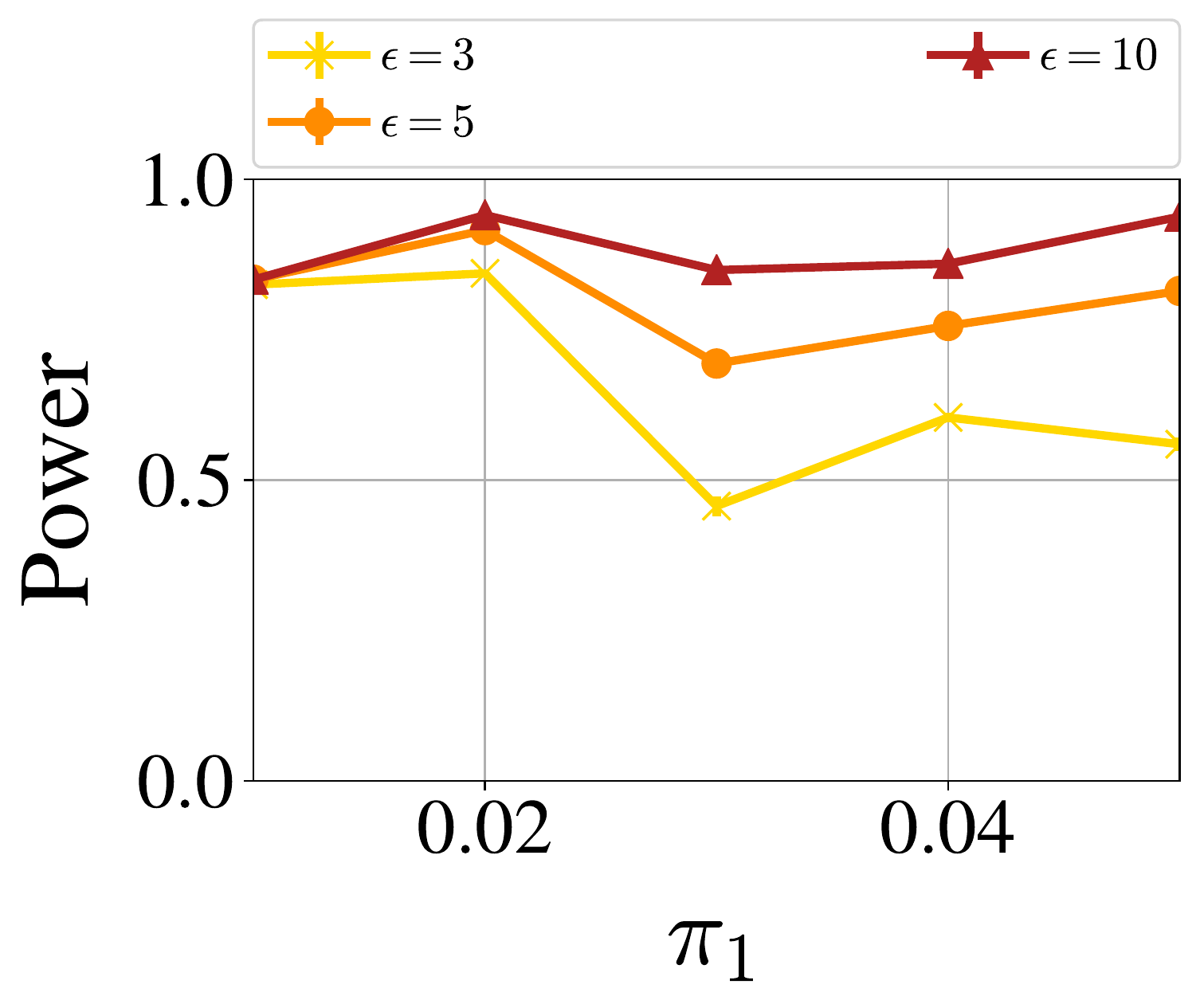}}\\\vspace{-0.2cm}
			\subfloat[][PAPRIKA]{\includegraphics[width=.45\textwidth]{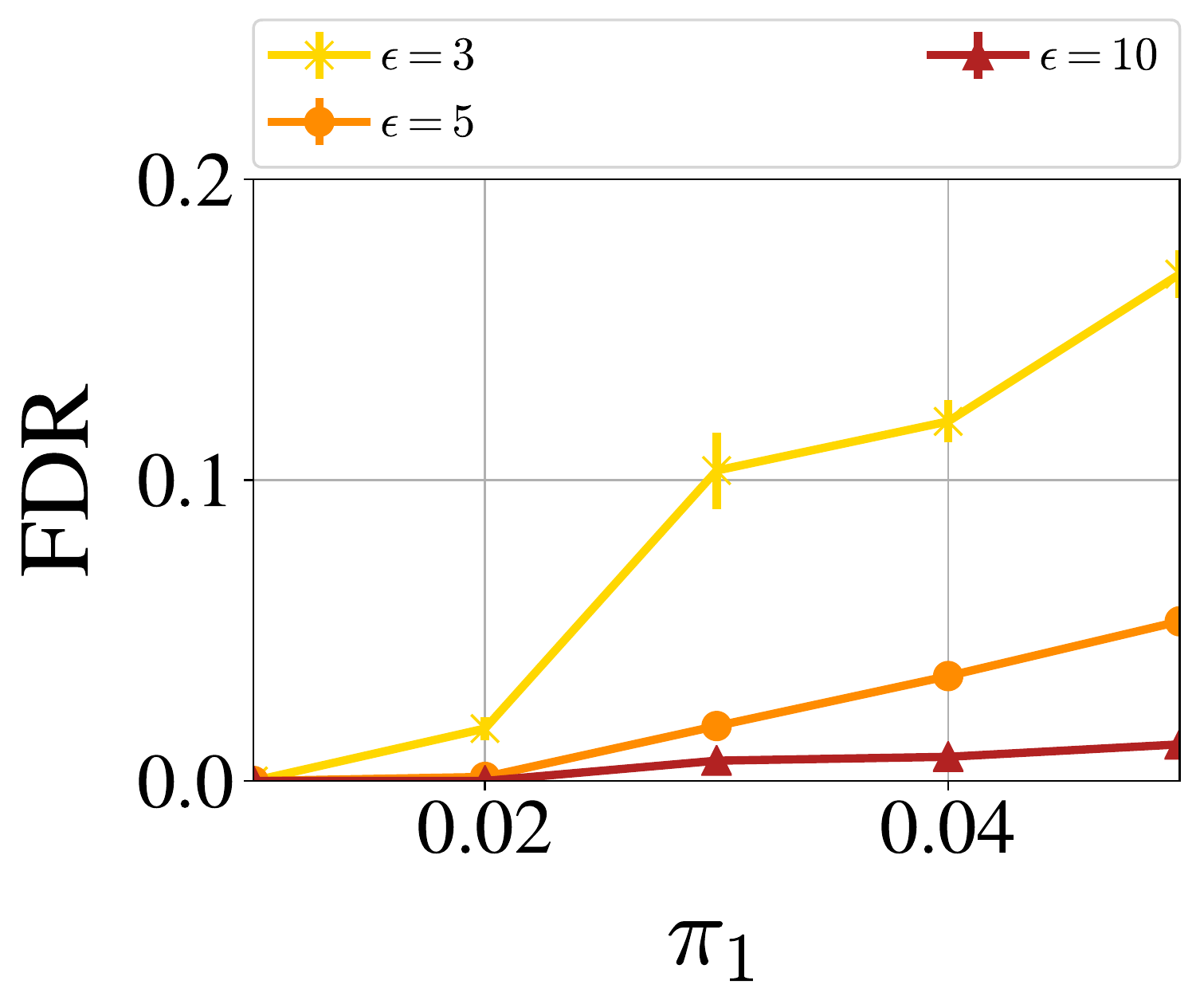}} 
			\subfloat[][PAPRIKA]{\includegraphics[width=.45\textwidth]{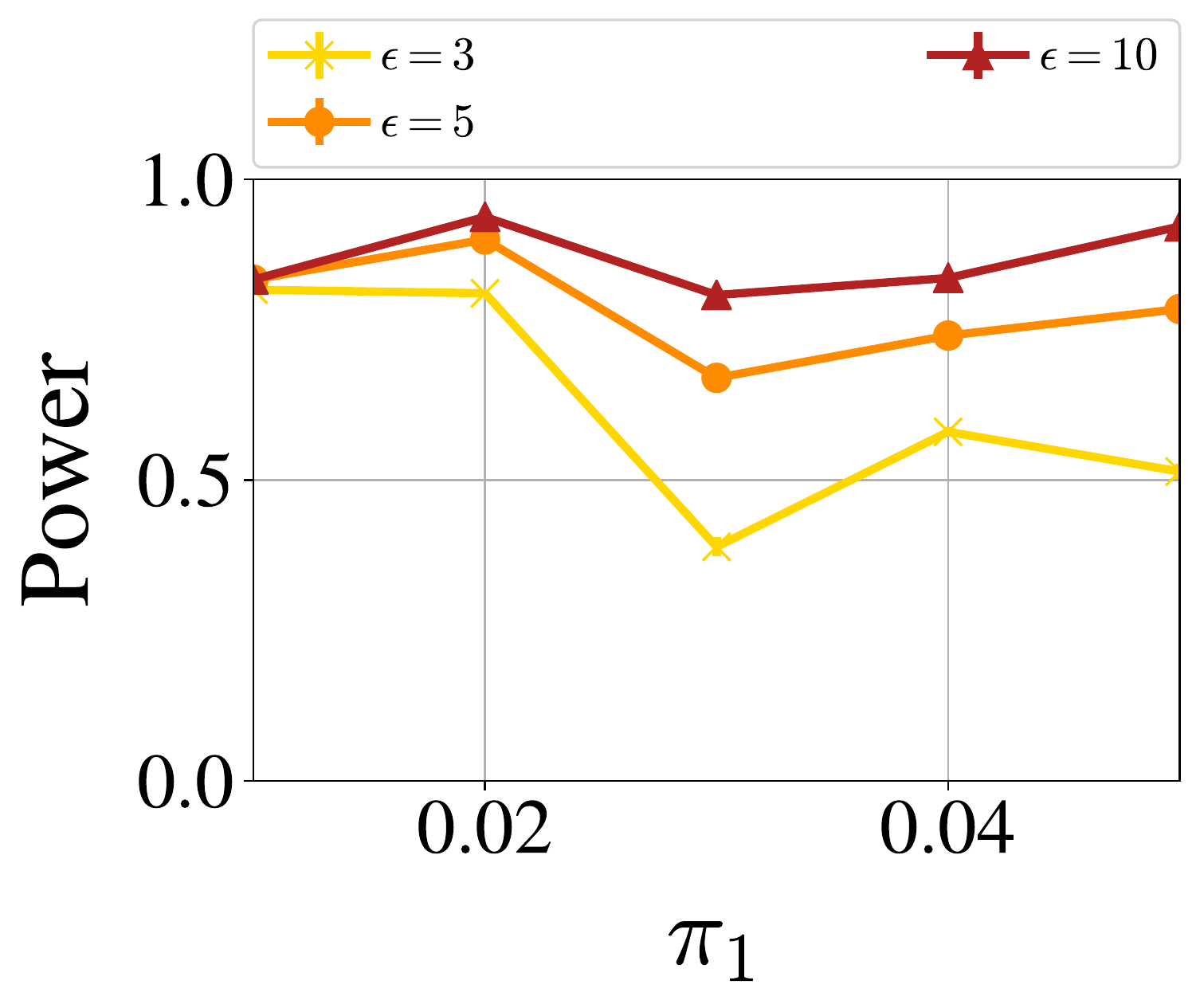}}
			\caption{\small FDR and statistical power versus fraction of non-null hypotheses $\pi_1$ for \PSAFFRON\ (with $\lambda_j=0.2$), \PSAFFRON\ AI (with $\lambda_j=\alpha_j$), and non-private algorithms  when the database consists of Bernoulli observations. 
			} \label{fig:example_bernoulli}
		\end{figure}
	\end{minipage}
\end{center}

As expected, the performance of \PSAFFRON\ generally diminishes as $\eps$ decreases. 
A notable exception is that FDR also decreases in Figure \ref{fig:example_bernoulli}(c). 
This phenomenon is because we set $\lambda_j=\alpha_j$, resulting in a smaller candidacy set and leading to insufficient rejections. 
Surprisingly, \PSAFFRON\ AI also yields a lower FDR than many of the non-private algorithms (Figure \ref{fig:example_bernoulli}(a)), since it tends to make fewer rejections.  We also see that \PSAFFRON\ AI performs dramatically better than \PSAFFRON, suggesting that the choice of $\lambda_j=\alpha_j$ should be preferred to constant $\lambda_j$ to ensure good performance in practice. 

\new{As \PSAFFRON\ is the first algorithm for private online FDR control, there is no private baseline for comparison. In Appendix~\ref{app.plots}, we show that na\"ive Laplace privatization plus \SAFFRON\ is ineffective.}


\vspace{-0.2cm}
\subsection{Testing with Truncated Exponential Observations}\label{sec:expon}
We again assume that the database $D$ contains $n$ individuals with $k$ independent features.
The $i$th feature is associated with $n$ i.i.d.\ truncated exponential distributed variables $\xi_1^i,\ldots,\xi_n^i$, each of which is sampled according to density 
\begin{equation*}
f_i(x \; | \; \theta_i,b)=\frac{\theta_i \exp(-\theta_i x)}{1-\exp(-b\theta_i)}I(0\le x \le b),
\end{equation*}
for positive parameters $b$ and $\theta_i$. 
Let $t_i$ be the realized sum of the $i$th features, and let $T_i$ denote the random variable of the sum of the $n$ truncated exponential distributed variables in the $i$th entry. A $p$-value for testing the null hypothesis $H_0^i: \theta_i=1$ against the alternative hypothesis $H_1^i: \theta_i > 1$ is given by,
\begin{equation*}
p_i(D)=\Pr_{\theta_i=1}(T_i>t_i).
\end{equation*}
\cite{dwork2018differentially} showed that $p_i$ is $(\mu,\eta)$-multiplicatively sensitive for $\mu=m^{-1-c}$ and $\eta\asymp\sqrt{\frac{\log n}{n}}$, where $m\le \text{poly}(n)$ and $c$ is any small positive constant. In the following experiments, we generate our database using the exponential distribution model truncated at $b=1$.
We set $\theta_i $ as follows:

\begin{equation*}
\theta_i=\begin{cases}
1 \quad \text{with probability } 1-\pi_1\\
1.95 \quad \text{with probability } \pi_1,
\end{cases}
\end{equation*}
where we vary the parameter $\pi_1$, corresponding to the expected fraction of non-nulls.

We sequentially test $H_0^i$ versus $H_1^i$ for $i=1,\ldots,k$. We use $n=1000$ as the size of the database $D$, and $k=800$ as the number of features as well as the number of hypotheses. While there is no closed form to compute the $p$-values, the sum of $n=1000$ i.i.d.\ samples is approximately normally distributed by the Central Limit Theorem. 
The expectation and the variance of $\xi^i_j$ with $b=1$ are 
\begin{align*}
\E{\xi^i_j}&=\frac{1}{\theta_i}+\frac{1}{1-\exp(\theta_i)}, \text{ and } \\
\text{Var}[\xi^i_j]&=\frac{1}{\theta_i^2}-\frac{\exp(\theta_i)}{(\exp(\theta_i)-1)^2},
\end{align*}
respectively. 
Therefore, $T_i$ is approximately distributed as $\mathcal{N}(n\E{\xi^i_j},n\text{Var}[\xi^i_j])$, and we compute the $p$-values accordingly.
We run the experiments with shift $A=\frac{c\eta}{\eps} \log\frac{2}{3\min\{\delta,1-((1-\delta)/\exp(\eps))^{1/k}\}}$ (shift magnitude $s=1$). The results are shown in Figure \ref{fig:example_expon}, which plots the FDR and statistical power against the expected fraction of non-nulls, $\pi_1$.

\begin{center}
	\begin{minipage}{0.8\linewidth}
		\begin{figure}[H]
			\centering
			\subfloat[][$\epsilon=5$]{\includegraphics[width=.45\textwidth]{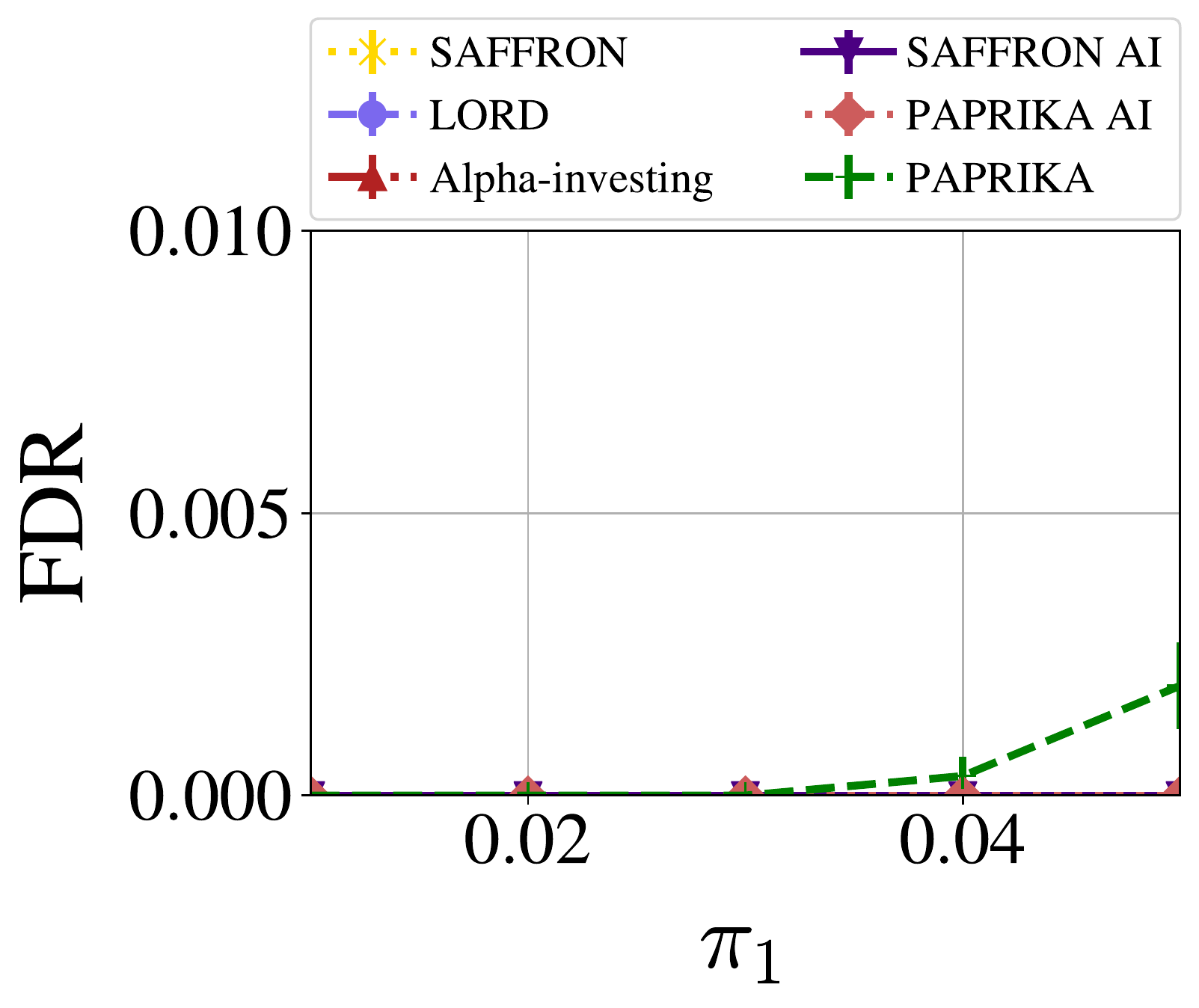}}
			\subfloat[][$\epsilon=5$]{\includegraphics[width=.45\textwidth]{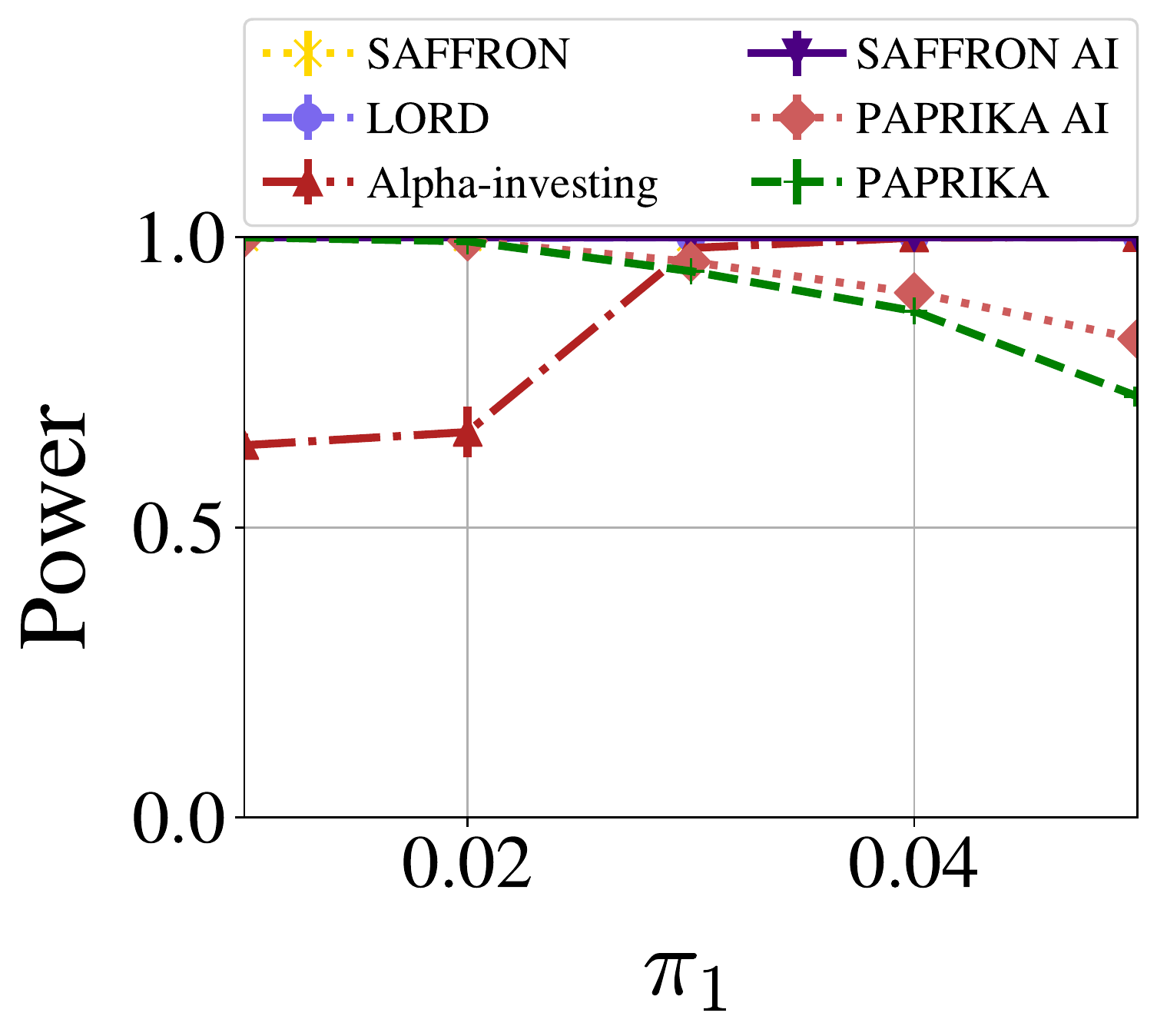}}\\
			\subfloat[][PAPRIKA AI]{\includegraphics[width=.45\textwidth]{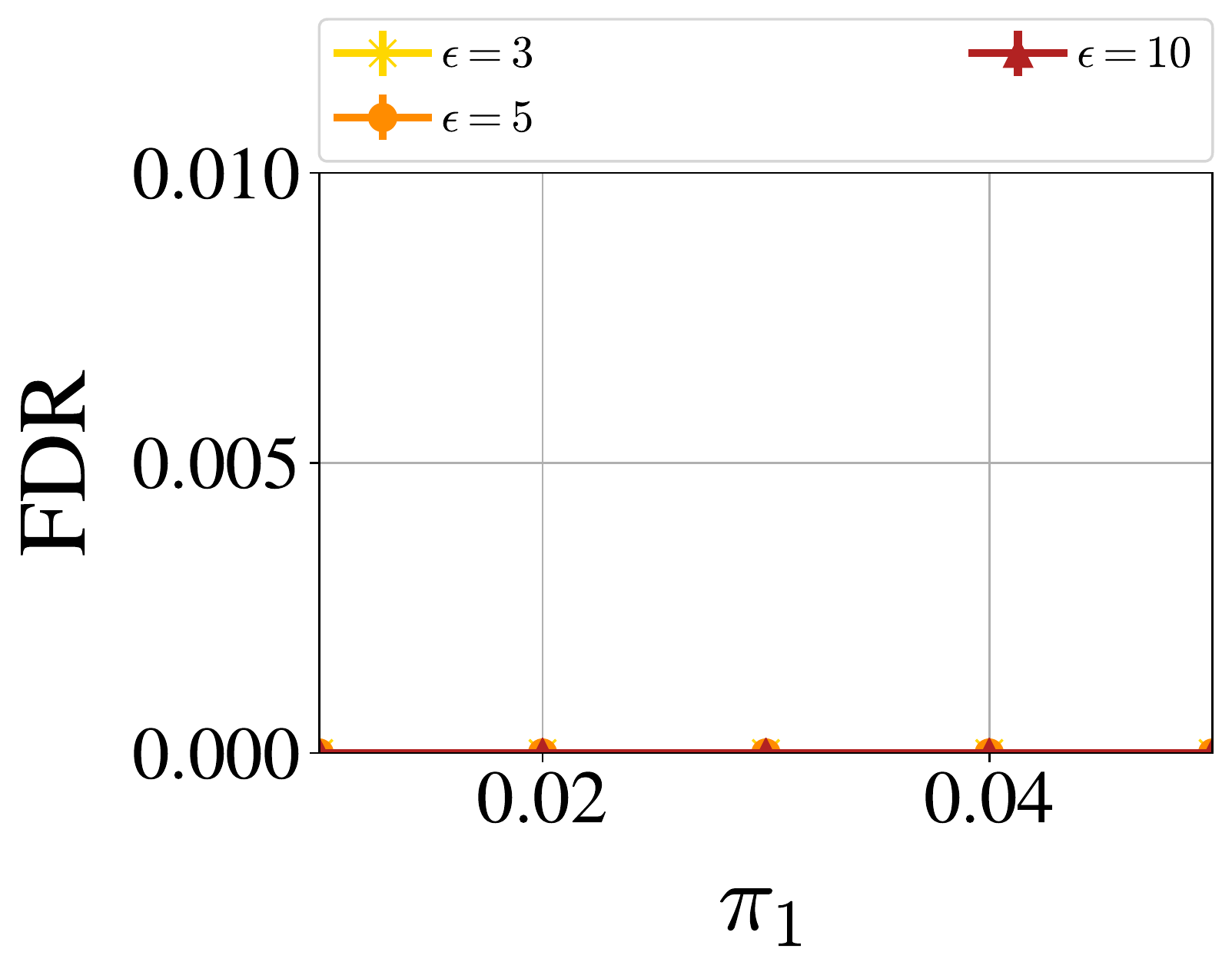}}
			\subfloat[][PAPRIKA AI]{\includegraphics[width=.45\textwidth]{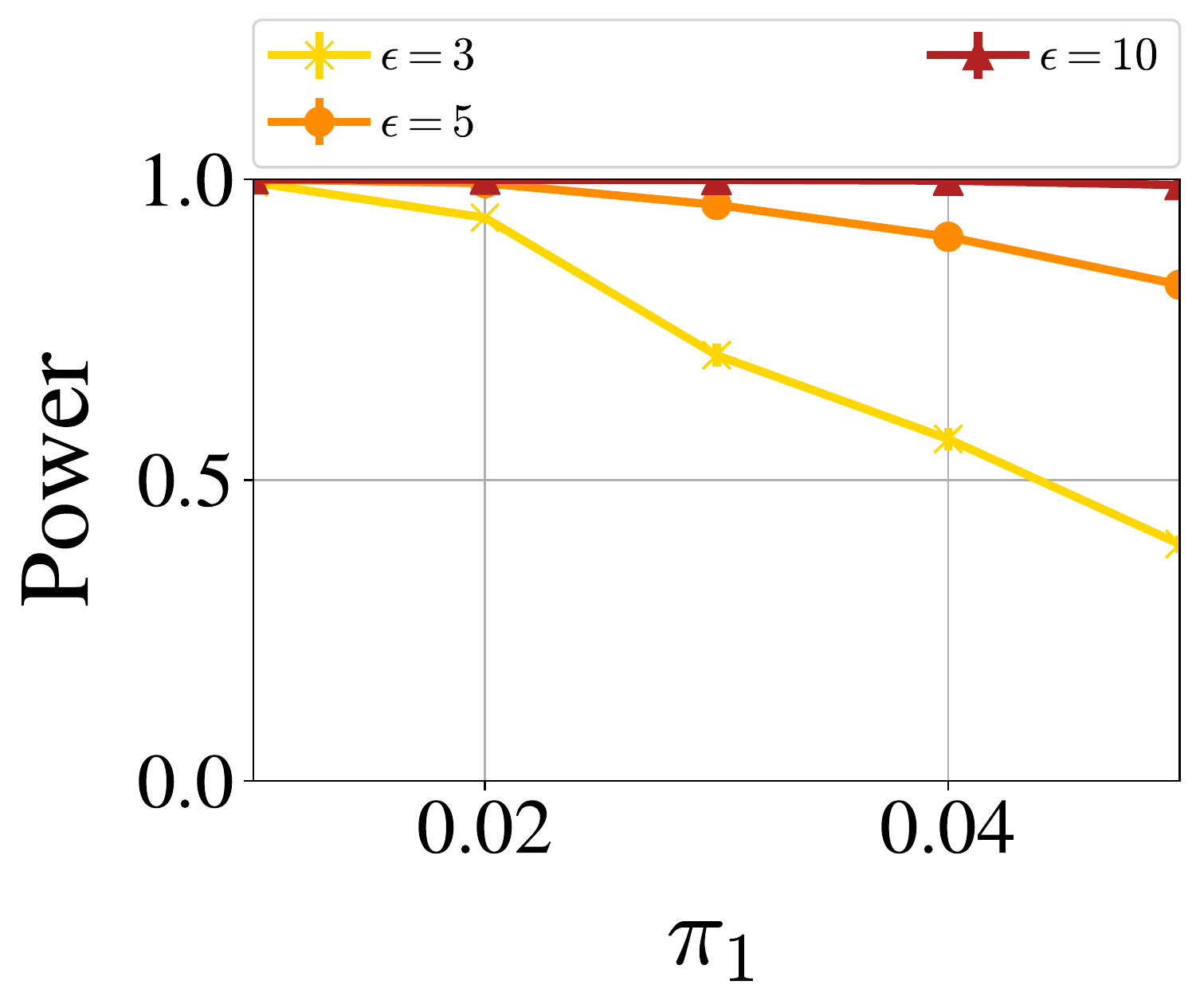}}\\
			\subfloat[][PAPRIKA]{\includegraphics[width=.45\textwidth]{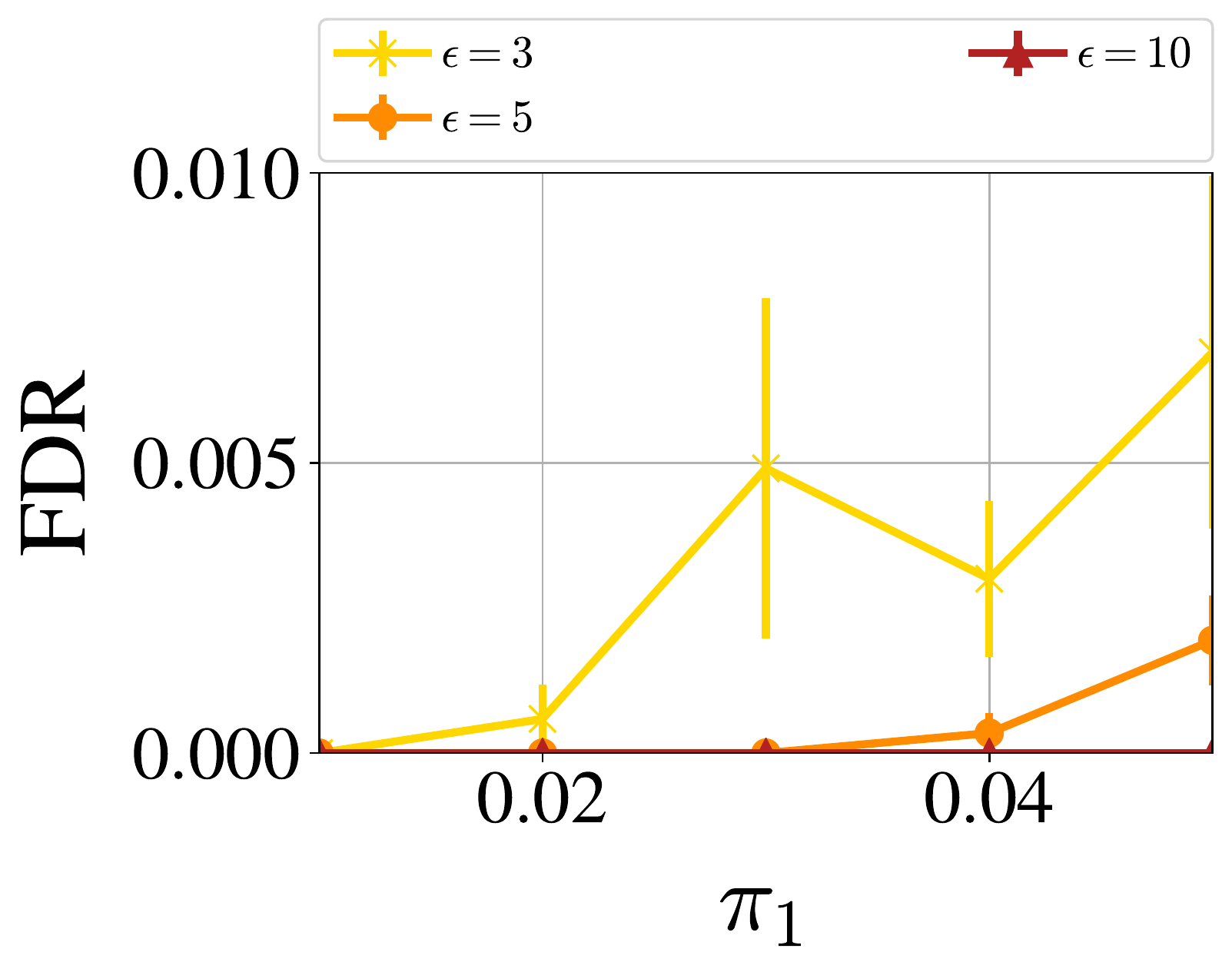}}
			\subfloat[][PAPRIKA]{\includegraphics[width=.45\textwidth]{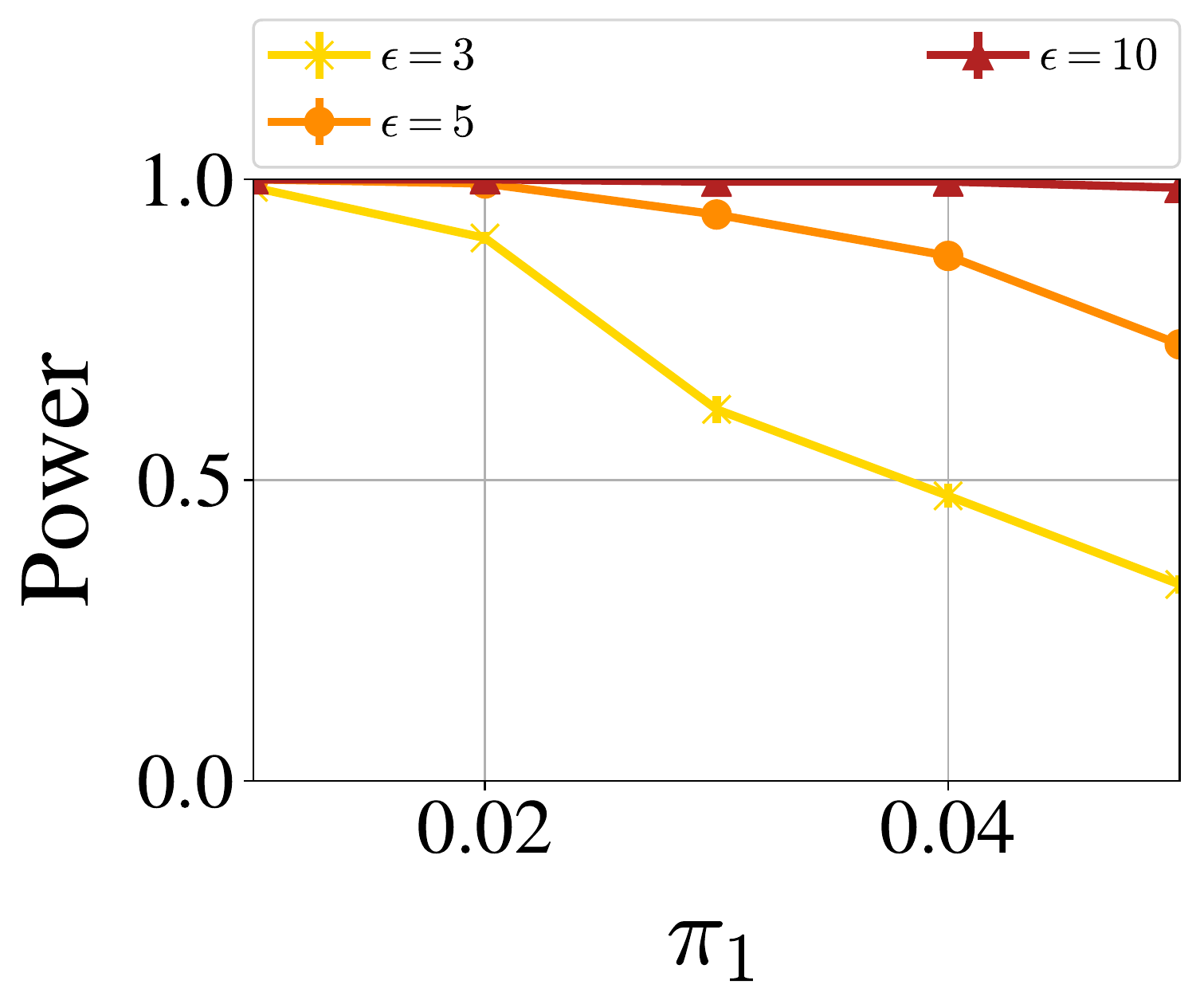}}
			\caption{\small FDR and statistical power versus fraction of non-nulls $\pi_1$ for \PSAFFRON\ (with $\lambda_j=0.2$), \PSAFFRON\ AI (with $\lambda_j=\alpha_j$), and non-private algorithms when the database consists of truncated exponential observations. 
			} \label{fig:example_expon}
		\end{figure}
	\end{minipage}
\end{center}

\new{As in the case with binomial data, we see that the performance of \PSAFFRON\ generally diminishes as $\epsilon$ decreases, and that \PSAFFRON\ AI outperforms \PSAFFRON, again reinforcing the need for tuning the parameters $\lambda_j$ based on the alpha-investing rule. All methods perform well in this setting, and the FDR of \PSAFFRON\ AI is visually indistinguishable from 0 at all levels of $\eps$ and $\pi_1$ tested. Numerical values are listed in Table \ref{table.exp} in Appendix \ref{app.plots} for ease of comparison.} 

We provide a further illustration of our experiments on truncated exponentials in Figure~\ref{fig:wealth_expon}.
In particular, we plot the rejection threshold $\alpha_t$ and wealth versus the hypothesis index.
Each ``jump'' of the wealth corresponds to a rejection. 
We observe that the rejections of our private algorithms are consistent with the rejections of the non-private algorithms, another perspective which empirically confirms their accuracy.

One hypothesis for the good performance observed in Figure \ref{fig:example_expon} is that the signal between the null and alternative hypotheses as parameterized by $\theta_i$ is very strong, meaning the algorithms can easily discriminate between the true null and true non-null hypotheses based on the observed $p$-values.  To measure this, we also varied the value of $\theta_i$ in the alternative hypotheses. These results are shown in Figure \ref{fig:signal}, which plots FDR and power of \PSAFFRON\ and \PSAFFRON\ AI with when the alternative hypotheses have parameter $\theta_i=1.90, 1.95, 2.00$. 
As expected, the performance gets better as we increase the signal, and we observe that when the signal is too weak ($\theta_i=1.90$), performance begins to decline.

For baseline of comparison, we include results for LapSAFFRON with $\epsilon = 5$, which is a na\"ive privatization of SAFFRON based on the Laplace Mechanism. For this baseline mechanism, LapSAFFRON first computes the $p$-values of each hypothesis, applies the Laplace Mechanism \cite{DMNS06} to the $p$-values, and then uses these noisy $p$-values as input to SAFFRON.  Overall privacy of the mechanism comes from advanced composition across multiple calls to the Laplace Mechanism, and post-processing guarantees of differential privacy, where the SAFFRON algorithm is post-processing on the privatized $p$-values.  We see that this baseline mechanism performs extremely poorly relative to \PSAFFRON\ and \PSAFFRON\ AI, motivating the need for our better algorithm design.


\begin{center}
	\begin{minipage}{0.8\linewidth}
		\begin{figure}[H]
			\centering
			\subfloat[][$\epsilon=5$]{\includegraphics[width=.4\textwidth]{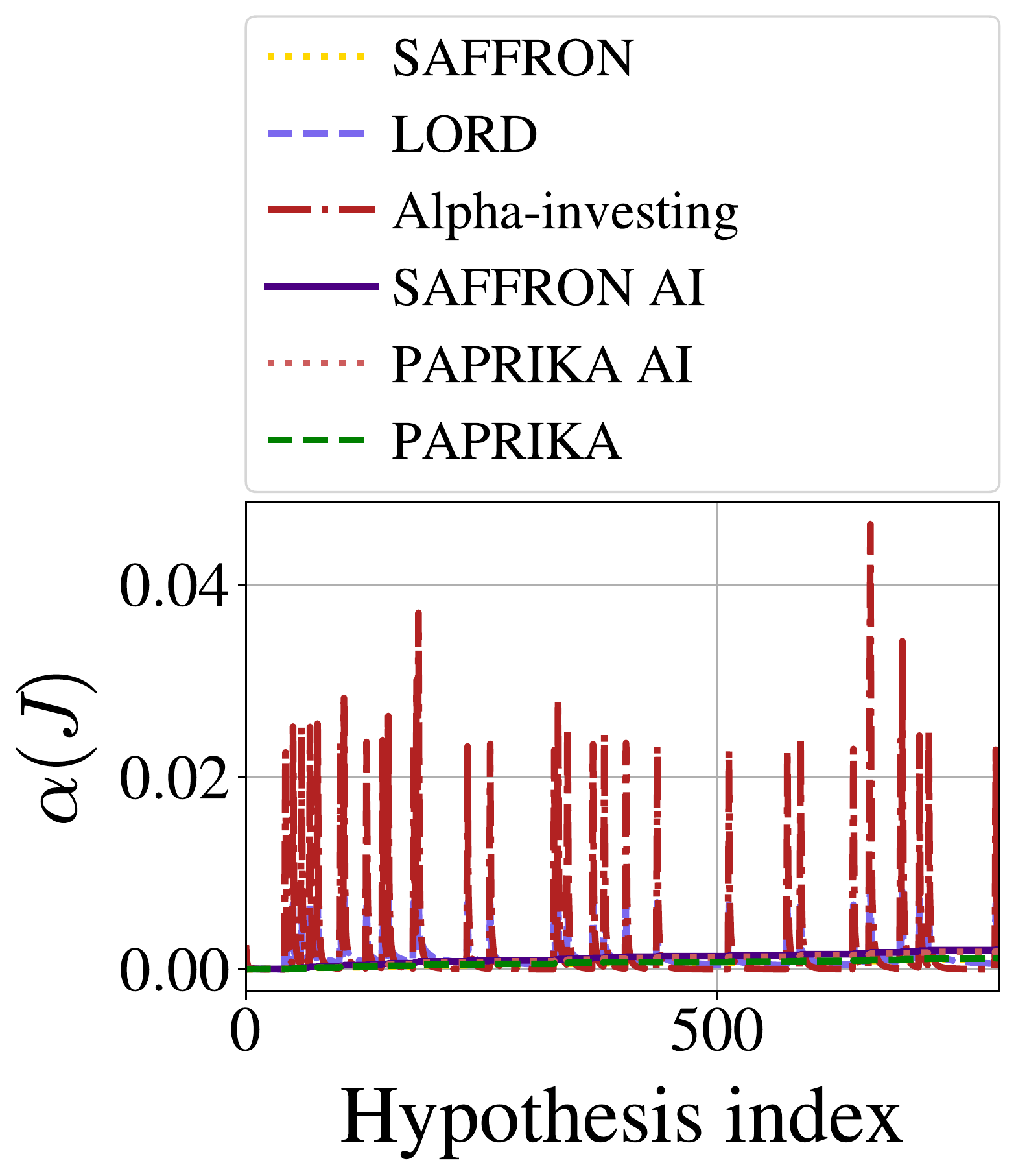}}
			\subfloat[][$\epsilon=5$]{\includegraphics[width=.4\textwidth]{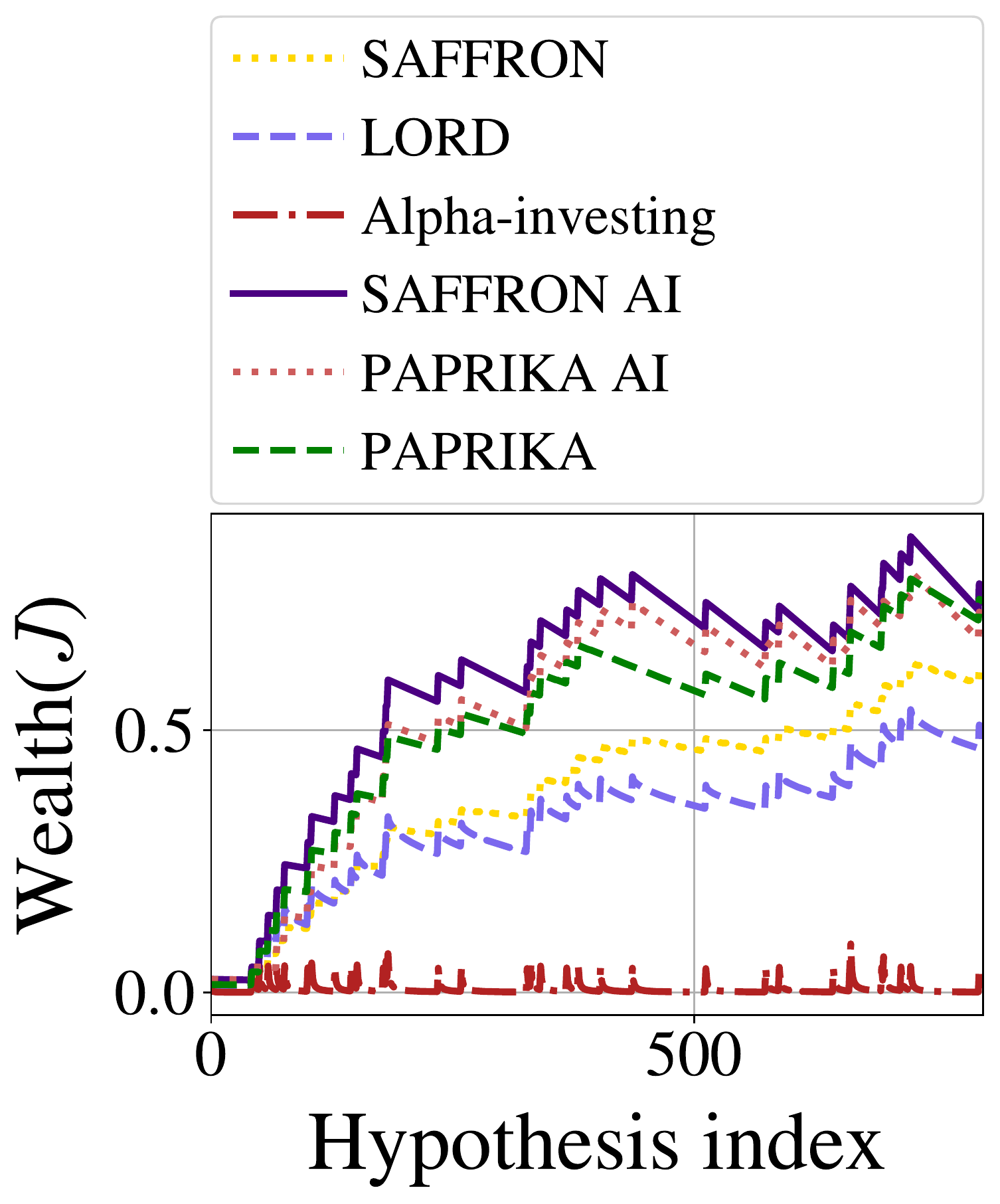}}
			\caption{\small Wealth and rejection threshold $\alpha_t$ versus hypothesis index with privacy parameter $\eps=5$ when the database consists of truncated exponential observations. \PSAFFRON\ AI and \SAFFRON\ AI used $\lambda_j=\alpha_j$, \PSAFFRON\ used $\lambda_j=0.2$, and \SAFFRON\ used $\lambda_j=0.5$. 
			} \label{fig:wealth_expon}
		\end{figure}
	\end{minipage}
\end{center}

\begin{center}
	\begin{minipage}{0.8\linewidth}
		\begin{figure}[H]
			\centering
			\subfloat[][PAPRIKA AI]{\includegraphics[width=.4\textwidth]{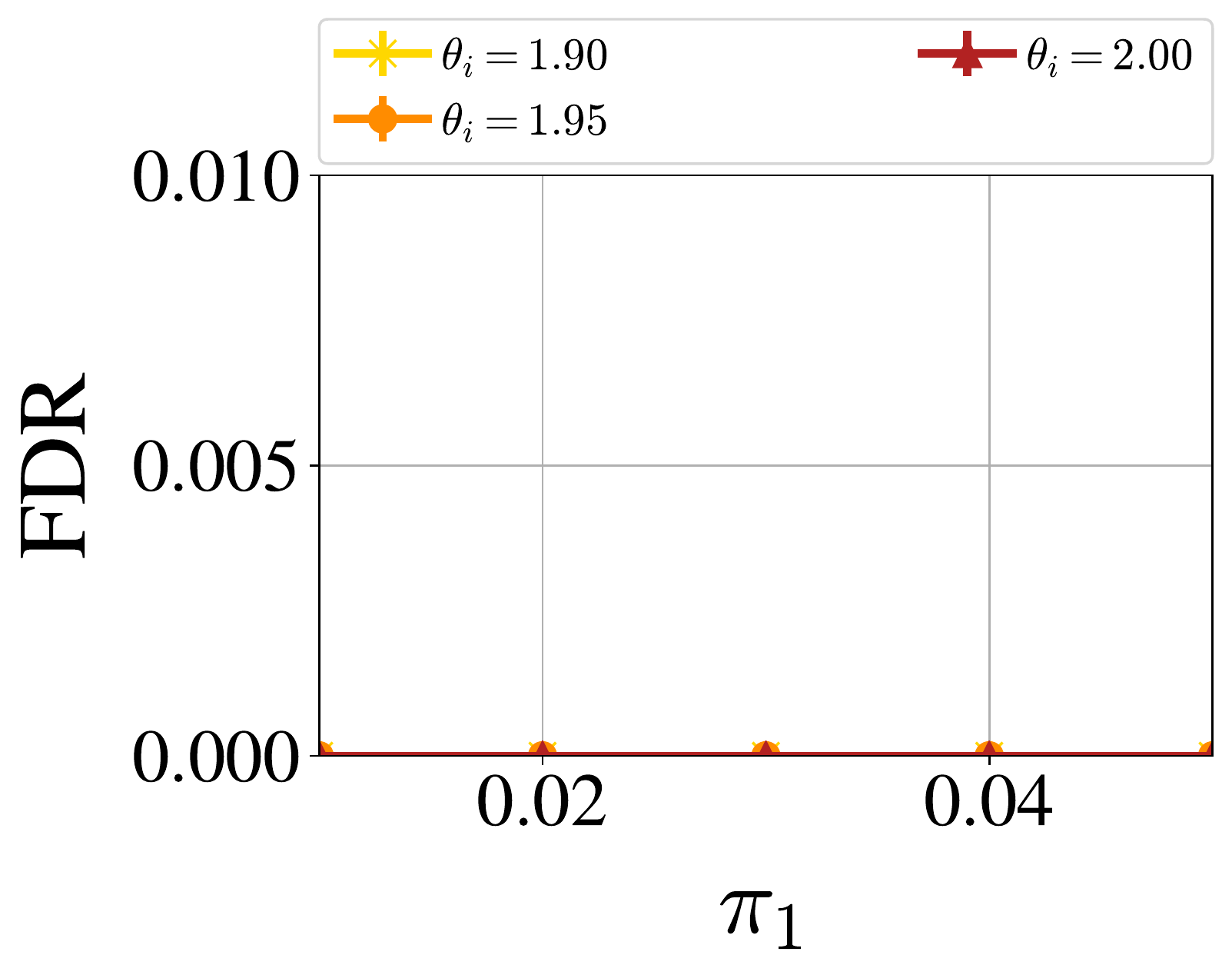}}
			\subfloat[][PAPRIKA AI]{\includegraphics[width=.4\textwidth]{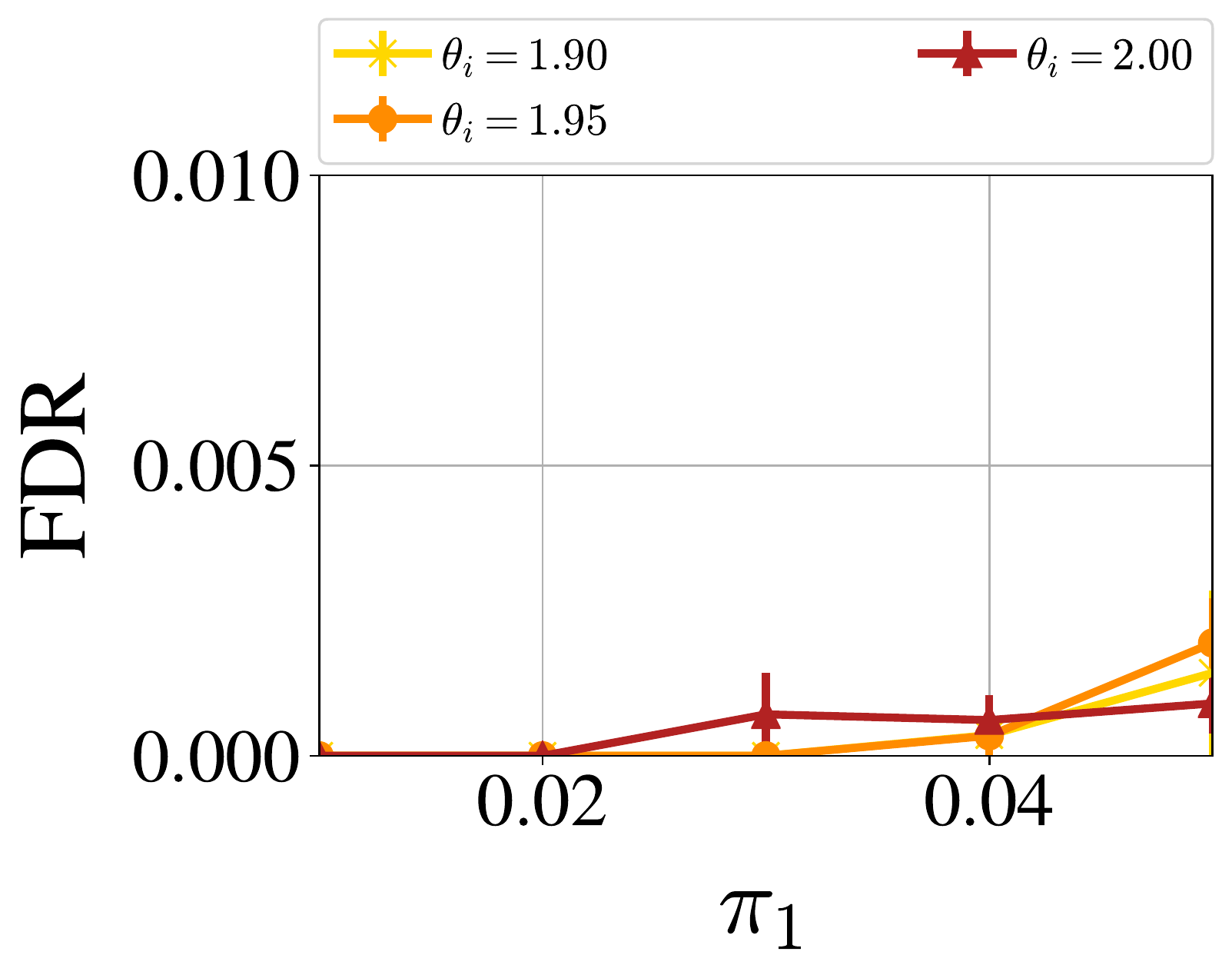}}\\
			\subfloat[][PAPRIKA]{\includegraphics[width=.4\textwidth]{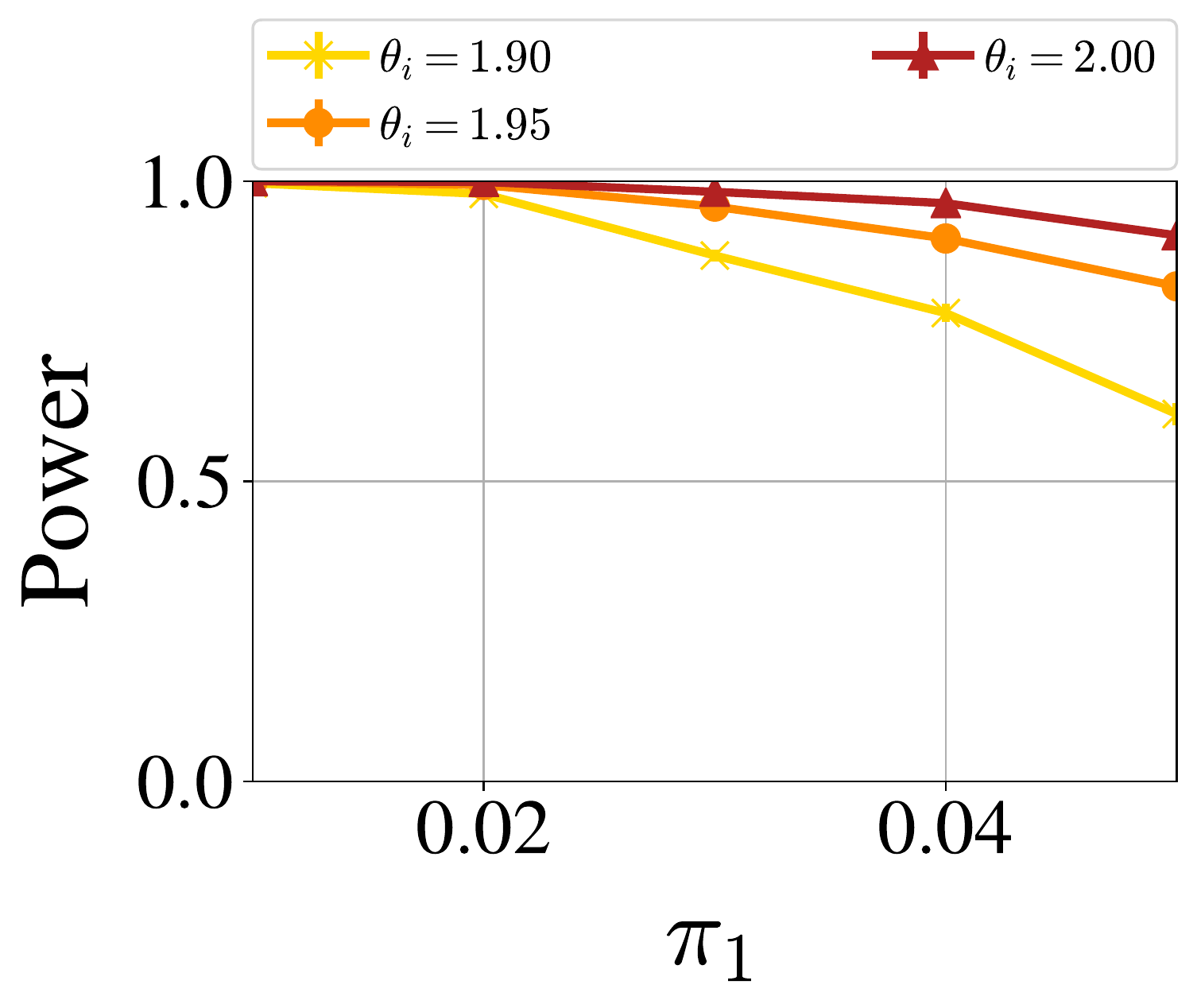}}
			\subfloat[][PAPRIKA]{\includegraphics[width=.4\textwidth]{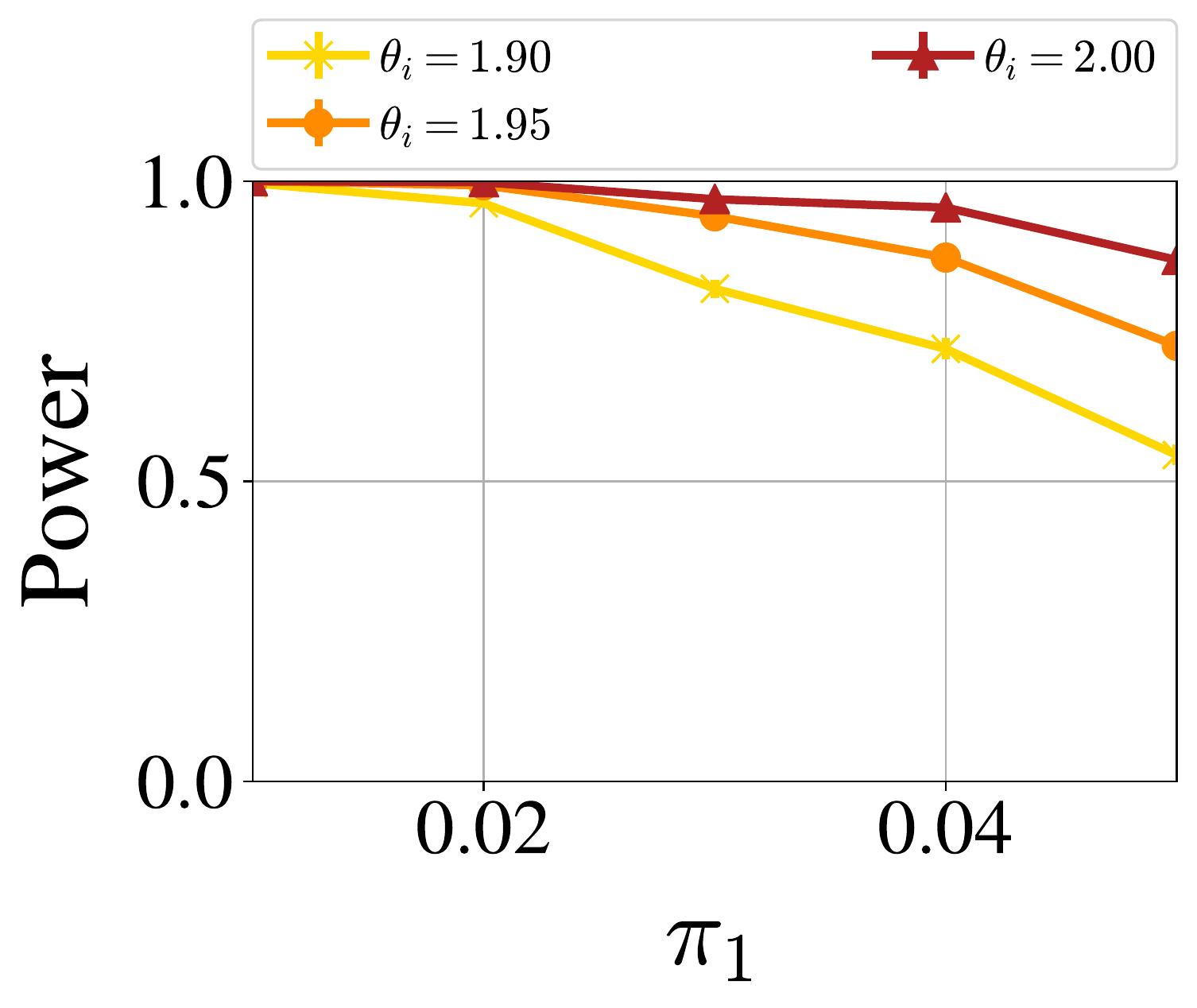}}
			\caption{\small FDR and statistical power versus expected fraction of non-null hypotheses $\pi_1$ under various choices of signal $\theta_i=1.90,1.95,2.00$  for alternative hypothesis parameters. The privacy parameter is $\eps=5$, and the database consists of truncated exponential observations. The first row shows performance of \PSAFFRON\ AI where $\lambda_j=\alpha_j$, and the second row shows performance of \PSAFFRON\ where $\lambda_j=0.2$.
			} \label{fig:signal}
		\end{figure}
	\end{minipage}
\end{center}

\subsection{Choice of shift $A$}\label{sec:shift}
We now discuss how to choose the shift parameter $A$. 
Theorem \ref{thm.priv} gives a theoretical lower bound for $A$ in terms of the privacy parameter $\delta$, but this bound may be overly conservative. 
Since the shift $A$ is closely related to the performance of FDR and statistical power, we wish to pick a value of $A$ that yields good performance in practice. 
In Theorem \ref{thm.fdr}, we show that FDR$(t)$ is less than our desired bound $\alpha$ plus the privacy parameter $\delta t$, which naturally requires that the privacy loss parameter $\delta$ be small. For a more detailed explanation, we bound Inequality~\eqref{eq.lap2} in the proof of Theorem \ref{thm.fdr} using Inequality~\eqref{eq.case2} from the proof of Theorem \ref{thm.priv}, and therefore, the empirical $\delta$ is naturally tied to the empirical FDR. As long as we can guarantee the empirical FDR to be bounded by the target FDR level, our privacy loss is bounded by the nominal $\delta$.

We use the Bernoulli example in Section \ref{sec:bernoulli} to investigate the performance under different choices of the shift $A$ with privacy parameter $\eps=5$. 
The results are summarized in Figure \ref{fig:choiceA}, which plots the FDR and power versus the expected fraction of non-nulls when we vary the shift size with $s=0.5,1,1.5,2$.

Larger shifts (corresponding to larger values of $s$) will lower the rejection threshold, which causes fewer hypotheses to be rejected.  This improves FDR of the algorithm, but harms Power, as the threshold may be too low to reject true nulls. Figure \ref{fig:choiceA} shows that the shift size parameter $s$ should be chosen by the analyst to balance the tradeoff between FDR and Power, as demanded by the application.

\begin{center}
	\begin{minipage}{0.8\linewidth}
		\begin{figure}[H]
			\centering
			\subfloat[][PAPRIKA AI]{\includegraphics[width=.4\textwidth]{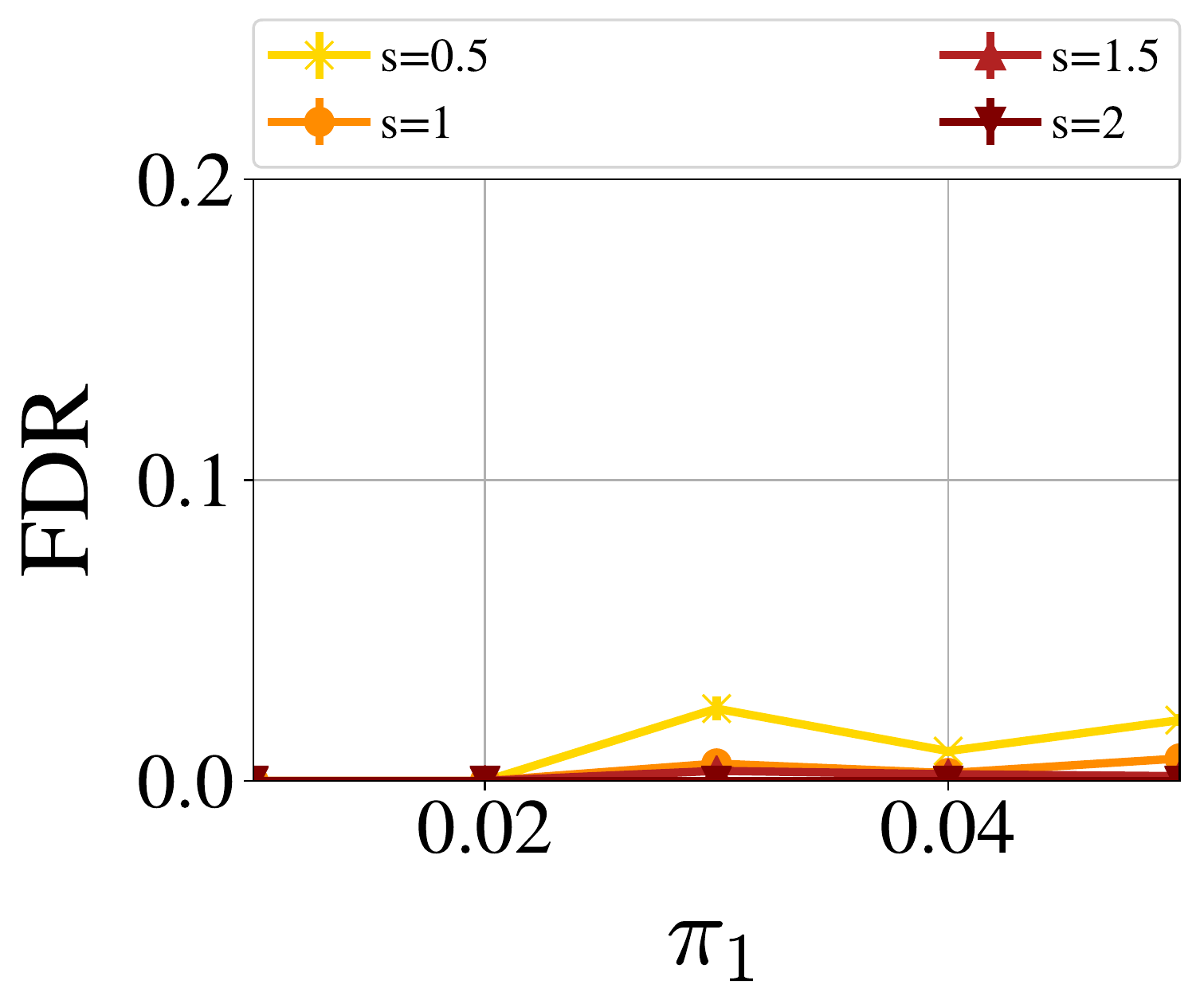}}
			\subfloat[][PAPRIKA AI]{\includegraphics[width=.4\textwidth]{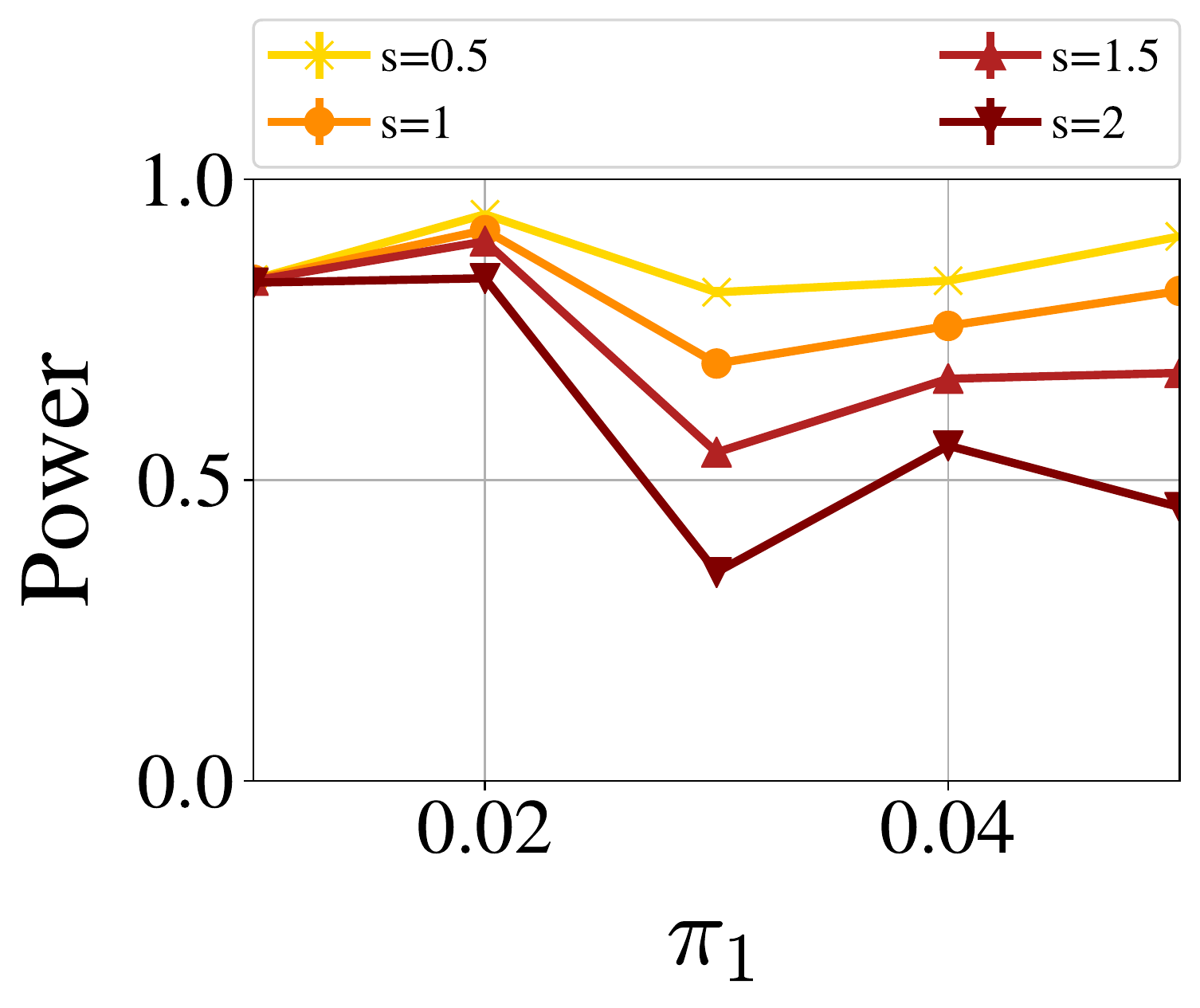}}\\
			\subfloat[][PAPRIKA]{\includegraphics[width=.4\textwidth]{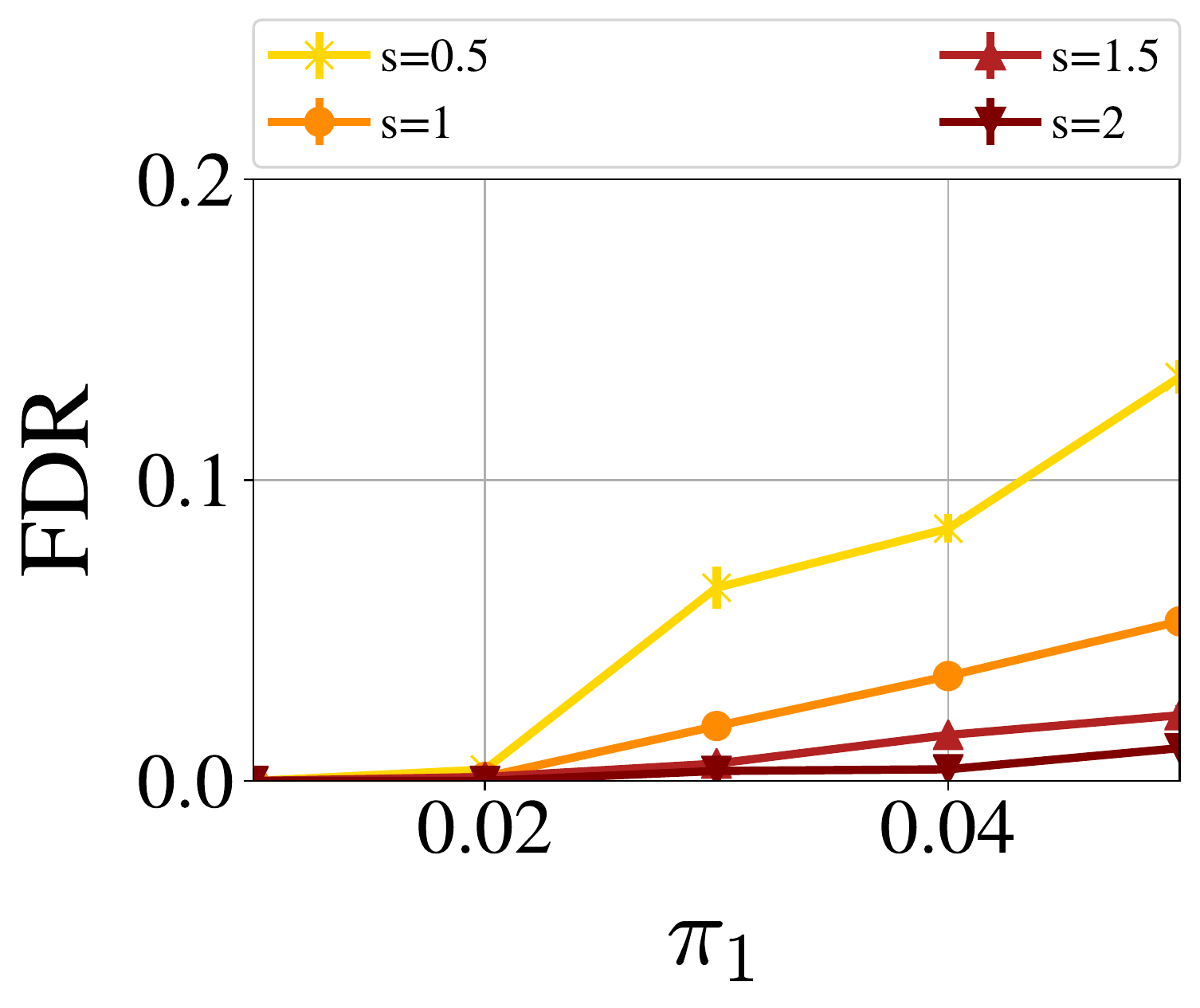}}
			\subfloat[][PAPRIKA]{\includegraphics[width=.4\textwidth]{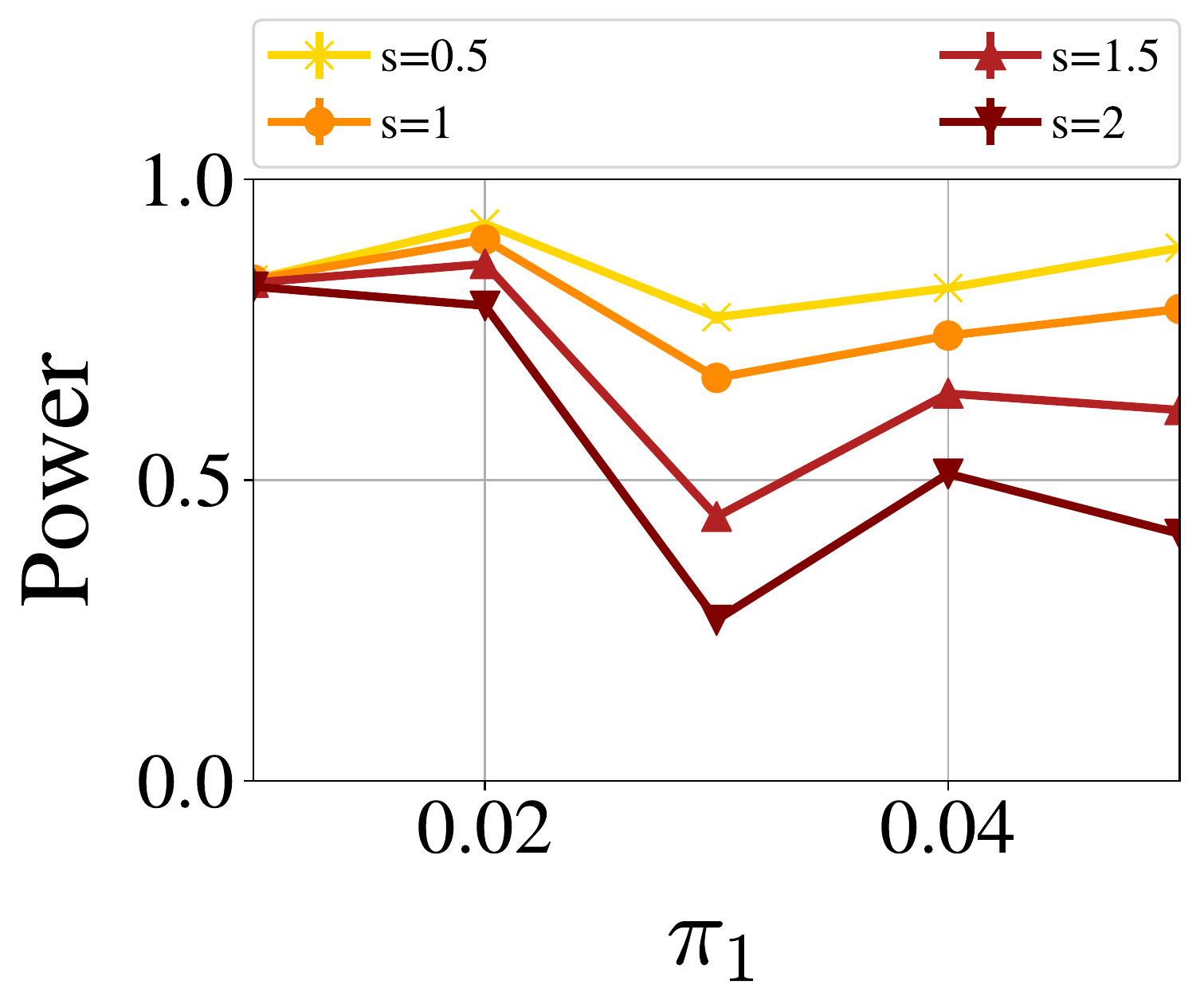}}
			\caption{\small FDR and statistical power versus expected fraction of non-null hypotheses $\pi_1$ under various choices of shift magnitude $s$. 
				The privacy parameter is $\eps=5$, and the database consists of Bernoulli observations. The first row shows performance of \PSAFFRON\ AI where $\lambda_j=\alpha_j$, and the second row shows performance of \PSAFFRON\ where $\lambda_j=0.2$.
			} \label{fig:choiceA}
		\end{figure}
	\end{minipage}
\end{center}

\bibliography{ref,biblio}
\bibliographystyle{alpha}

\newpage
\appendix

\section{Additional Tables}\label{app.plots}

Tables~\ref{table.bin} and~\ref{table.exp} report the numerical values for our experiments on Bernoulli and truncated exponential data, respectively. This information is also presented visually in Figures~\ref{fig:example_bernoulli} and~\ref{fig:example_expon}.
\begin{table}[H]
	\centering
	\resizebox{\textwidth}{!}{
		\begin{tabular}{@{}l |l |  c c |  c  c |  c c |  c  c |  c c |  c  c | c c @{}}
			\toprule
			\multirow{2}{*}{$\pi$} &\multirow{2}{*}{$\eps$}   & \multicolumn{2}{c|}{PAPRIKA AI} & \multicolumn{2}{c|}{PAPRIKA} & \multicolumn{2}{c|}{SAFFRON AI} & \multicolumn{2}{c|}{SAFFRON}& \multicolumn{2}{c|}{LORD} & \multicolumn{2}{c|}{Alpha-investing} & \multicolumn{2}{c}{LapSAFFRON}  \\ \cmidrule(r){3-16}
			&   & FDR              & power             & FDR           & power          & FDR              & power             & FDR           & power  & FDR              & power             & FDR           & power   & FDR           & power        \\ \midrule \midrule
			
			\multirow{3}{*} { 0.01}     &3          &   0               &    .825                  &    0          &   .817               &   \multirow{3}{*} {0}              &    \multirow{3}{*} {.833}                  &     \multirow{3}{*} {0}           &   \multirow{3}{*} {.833}  &   \multirow{3}{*} {0}               &    \multirow{3}{*} {.833}                  &     \multirow{3}{*} {0}           &   \multirow{3}{*} {.833}    &     \multirow{3}{*} {.990}           &   \multirow{3}{*} {.485}                                   \\      &5        &   0               &    .833                  &    0          &   .833               &                 &                     &                &     &                  &                      &                &               &             &             \\        &10       &0                  &.833                     &0              &.833                 &                  &                      &               &     &                  &                     &                &        &      &                          \\ \midrule			       
			
			\multirow{3}{*} { 0.02}      &3         &   0               &    .844                  &    .017          &   .810               &   \multirow{3}{*} {0}               &    \multirow{3}{*} {.938}                  &     \multirow{3}{*} { 0}           &   \multirow{3}{*} { .938}  &   \multirow{3}{*} { 0}               &    \multirow{3}{*} { .938}                  &     \multirow{3}{*} { 0}           &  \multirow{3}{*} {  .875}    &     \multirow{3}{*} {.973}           &  \multirow{3}{*} {  .509}                                  \\        &5      &   0               &    .916                  &    .001          &   .900               &                  &                      &                &     &                  &                     &                &               &            &              \\        &10       &   0               &    .941                  &    0          &   .938               &                  &                     &                &     &                 &                      &                &          &   &                            \\ \midrule
			
			\multirow{3}{*} { 0.03}        &3       &   .008               &    .457                  &    .103          &   .389               &   \multirow{3}{*} {.077}               &    \multirow{3}{*} {.923}                  &     \multirow{3}{*} { 0}           &   \multirow{3}{*} { .846}  &   \multirow{3}{*} { 0}               &    \multirow{3}{*} { .846}                  &     \multirow{3}{*} { 0}           &  \multirow{3}{*} {  .692}   &     \multirow{3}{*} {.977}           &  \multirow{3}{*} {  .509}                                   \\     &5         &   .006               &    .694                  &    .018          &   .670               &                  &                      &                &     &                  &                     &                &             &        &                    \\       &10        &   .015               &    .849                  &    .007          &   .808               &                  &                     &                &     &                 &                      &                &         &               &                 \\ \midrule			        
			
			\multirow{3}{*} { 0.04}     &3          &   .003               &    .604                  &    .120          &   .580               &   \multirow{3}{*} {.030}               &    \multirow{3}{*} {.970}                  &     \multirow{3}{*} { 0}           &   \multirow{3}{*} { .879}  &   \multirow{3}{*} { 0}               &    \multirow{3}{*} { .940}                  &     \multirow{3}{*} { 0}           &  \multirow{3}{*} {  .848}       &     \multirow{3}{*} {.943}           &  \multirow{3}{*} {  .512}                                \\    &5          &   .003               &    .756                  &    .035          &   .740               &                  &                      &                &     &                  &                     &                &          &      &                         \\     &10          &   .060               &    .860                  &    .008          &   .836              &                  &                     &                &     &                 &                      &                &         &         &                       \\ \midrule
			
			\multirow{3}{*} { 0.05}      &3         &   .009               &    .560                  &    .168          &   .514               &   \multirow{3}{*} {.056}               &    \multirow{3}{*} {.971}                  &     \multirow{3}{*} { .056}           &   \multirow{3}{*} { .971}  &   \multirow{3}{*} { .105}               &    \multirow{3}{*} { .971}                  &     \multirow{3}{*} { .056}           &  \multirow{3}{*} {  .971}       &     \multirow{3}{*} { .940}           &  \multirow{3}{*} {  .505}                               \\      &5        &   .007               &    .815                  &    .053          &   .785               &                  &                      &                &     &                  &                     &                &               &            &              \\       &10        &   .017               &    .938                  &    .012          &   .922              &                  &                     &                &     &                 &                      &                &        &            &                     \\ \bottomrule

	\end{tabular}}
	\vspace{0.1cm}
	\caption{Numerical values of FDR and power for Bernoulli observations experiments. LapSAFFRON corresponds to running SAFFRON on the na\"ive Laplace privatization of the p-values.}
	\label{table.bin}
	\vspace{-0.5cm}
\end{table}

\begin{table}[H]
	\centering
	\resizebox{\textwidth}{!}{
		\begin{tabular}{@{}l |l |  c c |  c  c |  c c |  c  c |  c c |  c  c|  c  c @{}}
			\toprule
			 \multirow{2}{*}{$\pi$} &\multirow{2}{*}{$\eps$}   & \multicolumn{2}{c|}{PAPRIKA AI} & \multicolumn{2}{c|}{PAPRIKA} & \multicolumn{2}{c|}{SAFFRON AI} & \multicolumn{2}{c|}{SAFFRON}& \multicolumn{2}{c|}{LORD} & \multicolumn{2}{c|}{Alpha-investing} & \multicolumn{2}{c}{LapSAFFRON} \\ \cmidrule(r){3-16}
			    &   & FDR              & power             & FDR           & power          & FDR              & power             & FDR           & power  & FDR              & power             & FDR           & power & FDR           & power         \\ \midrule \midrule
			       
			        \multirow{3}{*} { 0.01}     &3          &   0               &    .995                  &    0          &   .987               &   \multirow{3}{*} {0}              &    \multirow{3}{*} {1.00}                  &     \multirow{3}{*} {0}           &   \multirow{3}{*} {1.00}  &   \multirow{3}{*} {0}               &    \multirow{3}{*} {1.00}                  &     \multirow{3}{*} {0}           &   \multirow{3}{*} {.638}   &     \multirow{3}{*} {.989}           &   \multirow{3}{*} {.543}                                    \\      &5        &   0               &    1.00                  &    0          &   1.00               &                 &                     &                &     &                  &                      &                &               &           &               \\        &10       &0                  &1.00                     &0              &1.00                  &                  &                      &               &     &                  &                     &                &       &         &                        \\ \midrule			       
			
			        \multirow{3}{*} { 0.02}      &3         &   0               &    .936                  &    0          &   .903               &   \multirow{3}{*} {0}               &    \multirow{3}{*} {1.00}                  &     \multirow{3}{*} { 0}           &   \multirow{3}{*} { 1.00}  &   \multirow{3}{*} { 0}               &    \multirow{3}{*} { .999}                  &     \multirow{3}{*} { 0}           &  \multirow{3}{*} {  .676}    &     \multirow{3}{*} { .973}           &  \multirow{3}{*} {  .505}                                  \\        &5      &   0               &    .994                  &    0          &   .993               &                  &                      &                &     &                  &                     &                &               &        &                  \\        &10       &   0               &    .999                  &    0          &   1.00               &                  &                     &                &     &                 &                      &                &          &        &                       \\ \midrule
			        
			        \multirow{3}{*} { 0.03}        &3       &   0               &    .708                  &    .005          &   .618               &   \multirow{3}{*} {0}               &    \multirow{3}{*} {1.00}                  &     \multirow{3}{*} { 0}           &   \multirow{3}{*} { 1.00}  &   \multirow{3}{*} { 0}               &    \multirow{3}{*} { 1.00}                  &     \multirow{3}{*} { 0}           &  \multirow{3}{*} {  .982}   &     \multirow{3}{*} {.977}           &  \multirow{3}{*} {  .516}                                   \\     &5         &   0               &    .958                 &    0          &   .942               &                  &                      &                &     &                  &                     &                &               &       &                   \\       &10        &   0               &    .999                  &    0          &   .996               &                  &                     &                &     &                 &                      &                &           &           &                   \\ \midrule			        
			        
			        \multirow{3}{*} { 0.04}     &3          &   0               &    .569                  &    .003          &   .474               &   \multirow{3}{*} {0}               &    \multirow{3}{*} {1.00}                  &     \multirow{3}{*} { 0}           &   \multirow{3}{*} { 1.00}  &   \multirow{3}{*} { 0}               &    \multirow{3}{*} { 1.00}                  &     \multirow{3}{*} { 0}           &  \multirow{3}{*} {  .999}    &     \multirow{3}{*} {.944}           &  \multirow{3}{*} {  .503}                                   \\    &5          &   0               &    .905                  &    0          &   .873               &                  &                      &                &     &                  &                     &                &             &          &                  \\     &10          &   0               &    .998                  &    0          &   .996              &                  &                     &                &     &                 &                      &                &         &         &                      \\ \midrule
			      
			        \multirow{3}{*} { 0.05}      &3         &   0               &    .394                  &    .007          &   .327               &   \multirow{3}{*} {0}               &    \multirow{3}{*} {1.00}                  &     \multirow{3}{*} { 0}           &   \multirow{3}{*} { 1.00}  &   \multirow{3}{*} { 0}               &    \multirow{3}{*} { 1.00}                  &     \multirow{3}{*} { 0}           &  \multirow{3}{*} {  1.00}    &     \multirow{3}{*} {.940}           &  \multirow{3}{*} { .505}                                  \\      &5        &   0               &    .825                  &    .002          &   .726              &                  &                      &                &     &                  &                     &                &           &        &                      \\       &10        &   0               &    .990                  &    0          &   .986              &                  &                     &                &     &                 &                      &                &        &           &                      \\ \bottomrule

	\end{tabular}}
	\vspace{0.1cm}
	\caption{Numerical values of FDR and power for truncated exponential observations experiments. LapSAFFRON corresponds to running SAFFRON on the na\"ive Laplace privatization of the p-values.}
	\label{table.exp}
	\vspace{-0.7cm}
\end{table}

\vspace{0.5in}
\section{Proof of Theorem \ref{thm.priv}}\label{privacy.proofs}

Before proving Theorem \ref{thm.priv}, we will state and prove the following lemma, which will be useful in the proofs of Theorem \ref{thm.priv} and Theorem \ref{thm.fdr}.
\begin{lemma}\label{lemma2}
	If $Z_1\sim\Lap(2b)$, $Z_2\sim\Lap(b)$ and $C>0$ is a constant, we have $\Pr(Z_1\ge Z_2-C)=1-\frac{2}{3}\exp(-\frac{C}{2b})+\frac{1}{6}\exp(-C/b)$.
\end{lemma}

\begin{proof}
	\begin{align*}
	\Pr(Z_1\ge Z_2-C)&=\int_{-\infty}^{\infty}\int_{x-C}^{\infty}\frac{1}{4b}\exp(-\frac{|y|}{2b})\frac{1}{2b}\exp(-\frac{|x|}{b})dydx\\
	&=\int_{\infty}^{C}(1-\frac{1}{2}\exp(-\frac{|x-C|}{2b}))\frac{1}{2b}\exp(-\frac{|x|}{b})dx+\int_{C}^{\infty}\frac{1}{2}\exp(-\frac{|x-C|}{2b})\frac{1}{2b}\exp(-\frac{|x|}{b})dx\\
	&=\int_{-\infty}^C\frac{1}{2b}\exp(-\frac{|x|}{b})dx-\int_{-\infty}^0\frac{1}{4b}\exp(-\frac{|3x-C|}{2b})dx\\
	&-\int_0^C\frac{1}{4b}\exp(-\frac{C+x}{2b})dx+\int_C^\infty\frac{1}{4b}\exp(-\frac{|3x-C|}{2b})dx\\
	&=1-\frac{1}{2}\exp(-\frac{C}{b})-\frac{1}{6}\exp(-\frac{C}{2b})-\frac{1}{2}\exp(-\frac{C}{2b})+\frac{1}{2}\exp(-\frac{C}{b})+\frac{1}{6}\exp(-\frac{C}{b})\\
	&=1-\frac{2}{3}\exp(-\frac{C}{2b})+\frac{1}{6}\exp(-\frac{C}{b})
	\end{align*}
\end{proof}

\privthm*

\begin{proof}
	Fix any two neighboring databases $D$ and $D'$. Let $R$ denote the random variable representing the output of \PSAFFRON($D, \alpha, \lambda, W_0, \{\gamma_j\}_{j=0}^\infty, c, \epsilon, \delta,s$) and let $R'$ denote the random variable representing the output of \PSAFFRON($D', \alpha, \lambda, W_0, \{\gamma_j\}_{j=0}^\infty, c, \epsilon, \delta,s$). Let $k$ denote the total number of hypotheses. When $\log p_t \ge \log 2\lambda$ and $\log p'_t \ge \log 2\lambda$ for all $t$, $\Pr(R=\{0,0,\ldots,0\})=1=\Pr(R'=\{0,0,\ldots,0\})$. When $\log p_t < \log 2\lambda$ and $\log p'_t < \log 2\lambda$ for all $t$, privacy follows from the privacy of \SparseVector\ with dynamic thresholds. Since the threshold at each time $t$ only depends on the threshold at time $t-1$ and and private rejection $R(t-1)$, by post-processing, the threshold $\alpha_t$ is private. Then by post-processing and the privacy of \SparseVector\ , the rejection $R(t)$ is also private. We give the formal probability argument as follows.
	For any neighboring $D, D'$ and any sequence of hypotheses, we first consider the output up to the first rejection, which is \AboveThresh\ . Consider any output $r\in \{0, 1\}^l$. Let $r=\{r_1, r_2, \ldots, r_l\}$, with $r_l=1$ and $r_1=\ldots=r_{l-1}=0$. Let
	\begin{align*}
	f_i(D,z,\alpha_i)=\Pr(\log p_i(D)+Z_i < \log \alpha_i-A+z)\\
	g_i(D,z,\alpha_i)=\Pr(\log p_i(D)+Z_i \ge \log \alpha_i-A+z),
	\end{align*}
	where $\alpha_1,\ldots, \alpha_t$ is a fixed sequence of thresholds determined by the $r$. 
  We have 
	\begin{align}
	\frac{\Pr(R=r|D)}{\Pr(R'=r|D')}&=\frac{\int_{-\infty}^{\infty}\Pr(Z_\alpha=z)\Pr(R_l(D)=r_l|r_{l-1},\ldots,r_{1})\Pr(R_2(D)=r_2|r_1)\Pr(R_1(D)=r_1)dz}{\int_{-\infty}^{\infty}\Pr(Z_\alpha=z)\Pr(R_l(D')=r_l|r_{l-1},\ldots,r_{1})\Pr(R_2(D')=r_2|r_1)\Pr(R_1(D')=r_1)dz} \notag\\ 
	&=\frac{\int_{-\infty}^{\infty}\Pr(Z_\alpha=z)g_l(D,z, \alpha_l)\prod_{i=1}^{l-1}f_i(D,z,\alpha_i)dz}{\int_{-\infty}^{\infty}\Pr(Z_\alpha=z)g_l(D',z, \alpha_l)\prod_{i=1}^{l-1}f_i(D',z,\alpha_i)dz},	\notag\\
	&=\frac{\int_{-\infty}^{\infty}\Pr(Z_\alpha=z-\eta)g_l(D,z-\eta, \alpha_l)\prod_{i=1}^{l-1}f_i(D,z-\eta,\alpha_i)dz}{\int_{-\infty}^{\infty}\Pr(Z_\alpha=z)g_l(D',z, \alpha_l)\prod_{i=1}^{l-1}f_i(D',z,\alpha_i)dz}, \label{eq.changevar}	\\
	&\le\frac{\int_{-\infty}^{\infty}\exp(\eps/2c)\Pr(Z_\alpha=z)g_l(D,z-\eta, \alpha_l)\prod_{i=1}^{l-1}f_i(D',z,\alpha_i)dz}{\int_{-\infty}^{\infty}\Pr(Z_\alpha=z)g_l(D',z, \alpha_l)\prod_{i=1}^{l-1}f_i(D',z,\alpha_i)dz}, \label{ineq.1}	\\	
	&\le\frac{\int_{-\infty}^{\infty}\exp(\eps/2c)\Pr(Z_\alpha=z)\exp(\eps/2c)g_l(D',z, \alpha_l)\prod_{i=1}^{l-1}f_i(D',z,\alpha_i)dz}{\int_{-\infty}^{\infty}\Pr(Z_\alpha=z)g_l(D',z, \alpha_l)\prod_{i=1}^{l-1}f_i(D',z,\alpha_i)dz}, \label{ineq.2} \\
	&=\exp(\eps/c).
	\end{align}
    Equation \eqref{eq.changevar} is from change of integration variable $z$ to $z-\eta$. Inequality \eqref{ineq.1} is because $Z_\alpha$ follows $\Lap(2\eta c/\eps)$ and $\log p_i(D)-\eta \le \log p_i(D')$. Inequality \eqref{ineq.2} is because 
    \begin{align*}
    g_l(D,z-\eta, \alpha_l)&=\Pr(\log p_l(D)+Z_l \ge \log \alpha_l-A+z-\eta)\\
    &\le \Pr(\log p_l(D')+\eta +Z_l \ge \log \alpha_l-A+z-\eta)\\
    &\le \Pr(\log p_l(D')+Z_l \ge \log \alpha_l-A+z-2\eta)\\
    &\le \exp(\eps/2c) \Pr(\log p_l(D')+Z_l \ge \log \alpha_l-A+z)\\
    &\le \exp(\eps/2c) g_l(D',z, \alpha_l).
    \end{align*}
    When we restart \AboveThresh\ after the first rejection, the inital threshold is the post-processing of the previous ouputs, which is also private. Then by simple composition, the overall privacy loss is $\eps$.

	For other cases, the worst case is that for all $t$, $\log p_t < \log 2\lambda$ and $\log p'_t \ge \log 2\lambda$. In this setting, we have
	\begin{equation*}
	\Pr(R'=r)=\begin{cases} 1 &\text{if } r=\{0,0,\ldots,0\}\\
	0&\text{otherwise}.
	\end{cases}
	\end{equation*}
	 To satisfy $(\eps, \delta)$-differential privacy, we need to bound the probability of outputting $r$ for database $D$. We first consider $r=\{0,0\ldots,0\}$. We wish to bound $\Pr(R'=\{0,0\ldots,0\})\le \exp(\eps) \Pr(R=\{0,0,\ldots,0\})+\delta$ and $\Pr(R=\{0,0\ldots,0\})\le \exp(\eps) \Pr(R'=\{0,0,\ldots,0\})+\delta$. The latter is trivial since $\exp(\eps) \Pr(R'=\{0,0,\ldots,0\})+\delta=\exp(\eps)+\delta$, which is greater than 1. It remains to satisfy $\Pr(R'=\{0,0\ldots,0\})\le \exp(\eps) \Pr(R=\{0,0,\ldots,0\})+\delta$, which is equivalent to $1-\delta \le \exp(\eps) \Pr(R=\{0,0,\ldots,0\})$. We have
	 \begin{align}
	 \Pr(R=\{0,0\ldots,0\}) &= \Pr(R_1=0)\Pr(R_2=0|R_1=0)\ldots\Pr(R_k=0|R_{k-1}=0)\notag\\
	 &= \prod_{t=1}^{k}\Pr(\log p_t + Z_t \ge \log \alpha_t -A +Z_\alpha) \notag\\
	 &> \prod_{t=1}^{k}\Pr(\log 2\lambda-\eta+ Z_t \ge \log \alpha_t -A +Z_\alpha) \label{eq.worst1}\\
	 &=  \prod_{t=1}^{k}\Pr(Z_t \ge Z_\alpha + \log \alpha_t -\log 2\lambda +\eta-A) \notag\\
	 &= \prod_{t=1}^{k}\left( 1- \frac{2}{3}\exp(-\frac{\eps(A+\log(2\lambda/\alpha_t)-\eta)}{4\eta c})+\frac{1}{6}\exp(-\frac{\eps(A+\log(2\lambda/\alpha_t)-\eta)}{2\eta c}\right) \label{eq.p00}\\
	 &\ge \left( 1- \frac{2}{3}\exp(-\frac{\eps(A+\log2-\eta)}{4\eta c}) \right)^k, \label{eq.p0}
	 \end{align}
	 where Inequality \eqref{eq.worst1} is because the worst case happens when $p_t$ is $\eta$ below the candidacy threshold  $\log 2\lambda$, Equation \eqref{eq.p00} applies Lemma \ref{lemma2}, and Inequality \eqref{eq.p0} follows from the facts that $\alpha_t\le \lambda$ for all $t$ and that the third term in (\ref{eq.p00}) is positive. Setting \eqref{eq.p0} to be larger than $(1-\delta)/\exp(\eps)$, we have,
	 \begin{equation}\label{eq.case1}
	 \frac{2}{3}\exp\left(-\frac{\eps(A+\log2-\eta)}{4\eta c}\right)\le 1-\left( \frac{1-\delta}{\exp(\eps)}\right) ^{\frac{1}{k}}.
	 \end{equation}
	 
	 Next, we consider all other possible outputs $r$. 
   Define the set $S:=\{r\ |\ \text{there exists a } t \text{ such that } r_t=1\}$. We wish to bound $\Pr(R\in S)\le \exp(\eps)\Pr(R'\in S)+\delta$ and $\Pr(R' \in S)\le \exp(\eps)\Pr(R\in S)+\delta$. The latter is trivial since $\Pr(R'\in S)=0$. It remains to bound $\Pr(R\in S )\le \delta$. For any $t$, we have
	 \begin{align}
	 \Pr(R\in S)&\le \Pr(R_t=1) \notag\\
	 &=\Pr(\log p_t+Z_t\le \log\alpha_t-A+Z_\alpha) \notag\\
	 &\le \Pr(\log 2\lambda+Z_t\le \log \alpha_t-A+Z_\alpha) \label{eq.worst2}\\
	 &=\Pr(Z_t\le Z_\alpha-(\log(2\lambda/\alpha_t)+A)) \notag\\
	 &\le \Pr(Z_t\le Z_\alpha-(\log 2+A)) \notag\\
     &=\frac{2}{3}\exp\left(-\frac{\eps(A+\log 2)}{4\eta c}\right)-\frac{1}{6}\exp\left(-\frac{\eps(A+\log 2)}{2\eta c}\right) \label{eq.p10}\\
	 &\le \frac{2}{3}\exp\left(-\frac{\eps(A+\log 2)}{4\eta c}\right), \label{eq.p1}
	 \end{align}
	 where Inequality \eqref{eq.worst2} is because the worst case occurs when $\log p_t=\log 2\lambda$, Equality \eqref{eq.p10} applies Lemma \ref{lemma2}, and Inequality \eqref{eq.p1} follows from the facts that $\alpha_t\le \lambda$ for all $t$ and that the second term in (\ref{eq.p10}) is negative. 
   Setting (\ref{eq.p1}) to be less than $\delta$, we have,
	 \begin{equation}\label{eq.case2}
	 \frac{2}{3}\exp\left(-\frac{\eps(A+\log2)}{4\eta c}\right)\le \delta.
	 \end{equation}
	 Combining Equations \eqref{eq.case2} and \eqref{eq.case1}, we have the condition that  $\frac{2}{3}\exp\left(-\frac{\eps(A+\log2-\eta)}{4\eta c}\right)\le \min\{\delta,1-((1-\delta)/\exp(\eps))^{1/k}\}$. 

	 Rearranging this inequality for $A$ gives
	 \begin{equation*}
	 A\ge \frac{4\eta c}{\eps}\left( \log\frac{2}{3\min\{\delta,1-((1-\delta)/\exp(\eps))^{1/k}\}}-\log 2 +\eta\right), 
	 \end{equation*}
which is how the shift term $A$ is set in \PSAFFRON.
	 
\end{proof}

\section{Proof of Theorem \ref{thm.fdr}}
\label{sec.fdr-appendix}

\fdrthm*

\begin{proof}
	For any time $t>0$, before the total number of rejections reaches $c$ we bound the number of false rejections as follows:
	\begin{align}
	\E{|\mathcal{H}^0\cap\mathcal{R}(t)|}
     \le&\sum_{j\le t, j \in \mathcal{H}^0}\E{I(\log p_j+Z_j\le \log \alpha_j-A+Z_\alpha)}\label{eqn.11}\\
	\le& \sum_{j\le t, j \in \mathcal{H}^0}Pr(\log p_j\le \log \alpha_j)+Pr(Z_j\le Z_\alpha-A) \notag\\
	\le&\sum_{j\le t, j \in \mathcal{H}^0}\E{\alpha_j}+ Pr(Z_j\le Z_\alpha-A),\label{eqn.1}
	\end{align}
	where Inequality \eqref{eqn.11} follows from the rejection rule before the total number of rejections reaches $c$, and the number of false rejections is always $0$ afterwards. Inequality \eqref{eqn.1} follows from the conditional super-uniformity property. We bound each term in \eqref{eqn.1} separately. 
   Using the law of iterated expectations by conditioning on $\mathcal{F'}^{t-1}$, we can bound the first term of \eqref{eqn.1} as follows:
	 \begin{align}
	 \sum_{j\le t, j \in \mathcal{H}^0}\E{\alpha_j}
	\le& \E{\sum_{j\le t, j \in \mathcal{H}^0}\alpha_j \E{\frac{I(p_j>2\lambda_j)}{1-2\lambda_j}|\mathcal{F'}^{t-1}}}\notag\\
	=& \E{\E{\sum_{j\le t, j \in \mathcal{H}^0}\alpha_j\frac{I(p_j>2\lambda_j)}{1-2\lambda_j}|\mathcal{F'}^{t-1}}}\notag\\
	=& \E{\sum_{j\le t, j \in \mathcal{H}^0}\alpha_j\frac{I(p_j>2\lambda_j)}{1-2\lambda_j}},\label{eqn.falsereject}	 
	 \end{align}
	 where Equation \eqref{eqn.falsereject} applies the conditional super-uniformity. Since $\widehat{\text{FDP}}_{\text{\PSAFFRON}}(t)\le \alpha$, we have,
	\begin{equation*}
	\E{\sum_{j\le t, j \in \mathcal{H}^0}\alpha_j\frac{I(p_j>2\lambda_j)}{1-2\lambda_j}}\le \alpha \E{|\mathcal{R}(t)|}.
	\end{equation*}
	Next, we bound the second term in \eqref{eqn.1} as follows:
	\begin{align*}
	\sum_{j\le t, j \in \mathcal{H}^0}\Pr(Z_j\le Z_\alpha-A)
	\le&\frac{2t}{3}\exp\left(-\frac{A\epsilon}{4\eta c}\right)-\frac{t}{6}\exp\left(-\frac{A\epsilon}{2\eta c}\right)\\
	\le& t \min\left\lbrace \delta,1-\left( \frac{1-\delta}{\exp(\eps)}\right) ^{\frac{1}{k}}\right\rbrace.
	\end{align*}
    Combining this inequality with (\ref{eqn.falsereject}), we bound mFDR as 
	\begin{align*}
	mFDR&:=\frac{\E{|\mathcal{H}^0\cap\mathcal{R}(t)|}}{\E{|\mathcal{R}(t)|}}\\
	\le&\alpha + \frac{1}{\E{|\mathcal{R}(t)|}}\sum_{j\le t, j \in \mathcal{H}^0}\Pr(Z_j\le Z_\alpha-A)\\
	\le &\alpha+\min\left\lbrace \delta,1-\left( \frac{1-\delta}{\exp(\eps)}\right) ^{\frac{1}{k}}\right\rbrace  t \\
	\le &\alpha+\delta t.
	\end{align*}
	
	If the null $p$-values are independent of each other and the non-nullls, and $\{\alpha_t\}$ is a coordinate-wise non-decreasing function of the vector $R_1,\ldots, R_{t-1}$, then we have
	
	\begin{align}
	FDR(t)&=\E{\frac{|\mathcal{H}^0\cap\mathcal{R}(t)|}{|\mathcal{R}(t)|}} \notag\\
	&=\sum_{j\le t, j \in \mathcal{H}^0}\E{\frac{I(\log p_j+Z_j\le \log \alpha_j-A+Z_\alpha)}{|\mathcal{R}(t)|}}\notag\\
	&\le \sum_{j\le t, j \in \mathcal{H}^0}\E{\frac{\min\{\alpha_j\exp(Z_\alpha-Z_j-A),1\}}{|\mathcal{R}(t)|}}\label{eq.lemma1}\\
	&\le \sum_{j\le t, j \in \mathcal{H}^0}\E{\frac{\alpha_j}{|\mathcal{R}(t)|}}+\Pr(Z_j\le Z_\alpha-A)	\label{eq.independent},
	\end{align}
	where Inequality \eqref{eq.lemma1} applies the law of iterated expectations by conditioning on $\mathcal{F'}^{t-1}$ and Lemma \ref{lemma1}. 
  Inequality (\ref{eq.independent}) follows by a case analysis:
  if $Z_j>Z_\alpha-A$, then $\exp(Z_\alpha-Z_j-A)<1$, and thus $\frac{\min\{\alpha_j\exp(Z_\alpha-Z_j-A),1\}}{|\mathcal{R}(t)|}$ reduces to $\frac{\alpha_j}{|\mathcal{R}(t)|}$. 
  On the other hand, if $Z_j \leq Z_\alpha -A$, then $\frac{\min\{\alpha_j\exp(Z_\alpha-Z_j-A),1\}}{|\mathcal{R}(t)|}\le \frac{1}{|\mathcal{R}(t)|}\le 1$, allowing us to upper bound the expectation by the probability of this event.
  	
	We bound the first term in \eqref{eq.independent} as follows:
	\begin{align}
	\sum_{j\le t, j \in \mathcal{H}^0}\E{\frac{\alpha_j}{|\mathcal{R}(t)|}}
	&\le \sum_{j\le t, j \in \mathcal{H}^0}\E{\frac{\alpha_jI(p_j>2\lambda_j)}{(1-2\lambda_j)|\mathcal{R}(t)|}}\label{eq.lemma12}\\
	&\le \E{\frac{\sum_{j\le t}\alpha_jI(p_j>2\lambda_j)}{(1-2\lambda_j)|\mathcal{R}(t)|}}\notag\\
	&=\E{\widehat{\text{FDP}}_{\text{\PSAFFRON}}(t)}\notag\\
	&\le \alpha \label{eq.saffron2},
	\end{align}
	where Inequality \eqref{eq.lemma12} applies Lemma \ref{lemma1}.

  It remains to bound the second term in (\ref{eq.independent}), which we do using Lemma~\ref{lemma2} as follows:
	\begin{align}
	\sum_{j\le t, j \in \mathcal{H}^0}\Pr(Z_j\le Z_\alpha-A)&\le \sum_{j\le t}\Pr(Z_j\le Z_\alpha-A)\notag\\
	&=\frac{2t}{3}\exp(-\frac{A\epsilon}{4\eta c})-\frac{t}{6}\exp(-\frac{A\epsilon}{2\eta c})\notag\\
	&\le \min\left\lbrace \delta,1-\left( \frac{1-\delta}{\exp(\eps)}\right) ^{\frac{1}{k}}\right\rbrace t. \label{eq.lap2}
	\end{align}
	Combining \eqref{eq.saffron2} and \eqref{eq.lap2}, we reach the conclusion that $\text{FDR}(t)\le \alpha+\min\{\delta,1-((1-\delta)/\exp(\eps))^{1/k}\} t \le \alpha + \delta t$.  
	
\end{proof}

\section{Proof of Lemma \ref{lemma1}}
\label{sec.lemma-proof}

\lemfirst*

\begin{proof}
	The proof is similar to the proof of Lemma 2 in \cite{ramdas2018saffron} with the addition of i.i.d.~Laplace noise. 
	
	In a high level, we hallucinate a vector of $p$-values that are same as the original vector of $p$-values, except for the $t$-th index. 
  This allows us to apply the conditional uniformity property, since now $p_t$ is independent of the hallucinated rejections. 
  We then connect the original rejections and the hallucinated rejections by the monotonicity of the rejections.
	
    We perform our analysis using a hallucinated process: 
    let $\tilde{p}_{1:k}^{t}$ be a copy of $p_{1:k}$ that is identical everywhere except for the $t$-th $p$-value which is set to be 1. 
    That is, 
	\begin{equation*}
	\tilde{p}_i=\begin{cases} 1 &\text{if } i=t\\
	p_i &\text{otherwise}.
	\end{cases}
	\end{equation*}
	Also let the hallucinated Laplace noises $\tilde{Z}_{1:k}^t$ be an identical copy of $Z_{1:k}$, and let $\tilde{Z}_\alpha$ be an identical copy of $Z_\alpha$. 
  The $t$-th value of $\tilde{Z}_{1:k}^t$ can be arbitrary since we have ensure the event $\{\tilde{p}_t>2\lambda_t\}$, so it will fail to become a candidate and the values of $\tilde{Z}_t$ will not be relevant.
  We denote $\tilde{C}_{1:k}$ and $\tilde{R}_{1:k}$ as the candidates and rejections made using $\tilde{p}_{1:k}^{t}$, $\tilde{Z}_{1:k}^t$, and $\tilde{Z}_\alpha$.
	
    By construction, we have $\tilde{R}_{1:t-1}=R_{1:t-1}$. On the event $\{p_t>2\lambda_t\}$, we have $R_{t}=\tilde{R}_t=0$ and $C_{t}=\tilde{C}_t=0$ because $\tilde{p}_t=1$, so both will fail to become candidates, and hence we have $\tilde{R}_{1:k}=R_{1:k}$ and the following equation holds:
    \begin{equation*}
    	\frac{\alpha_tI(p_t>2\lambda_t)}{(1-2\lambda_t)h(R_{1:k})}=\frac{\alpha_tI(p_t>2\lambda_t)}{(1-2\lambda_t)h(\tilde{R}_{1:k})}.
    \end{equation*}
    We note that when $p_t\le2\lambda_t$, the above equation still holds since both sides will be zero. Since $\tilde{R}_{1:k}^{t}$ is independent of $p_t$, we have
    \begin{align}
    \E{\frac{\alpha_tI(p_t>2\lambda_t)}{(1-2\lambda_t)h(R_{1:k})}|\mathcal{F'}^{t-1}}&=\E{\frac{\alpha_tI(p_t>2\lambda_t)}{(1-2\lambda_t)h(\tilde{R}_{1:k})}|\mathcal{F'}^{t-1}}\notag\\
    &\ge \E{\frac{\alpha_t}{h(\tilde{R}_{1:k})}|\mathcal{F'}^{t-1}}\label{eq.superuni}\\
    &\ge \E{\frac{\alpha_t}{h(R_{1:k})}|\mathcal{F'}^{t-1}}\label{eq.tildeR}
    \end{align}
	where Inequality \eqref{eq.superuni} is obtained by taking the expectation only with respect to $p_t$ by invoking the conditional super-uniformity property and independence of $p_t$ and $h(\tilde{R}_{1:k})$, and Inequality \eqref{eq.tildeR} follows from the facts that $R_i\ge \tilde{R}_i$ for all $i$ and that the function $h$ is non-decreasing.
	
	For the second inequality in the lemma statement, we hallucinate a vector of $p$-values $\bar{p}_{1:k}^{t}$  that equals $p_{1:k}$ everywhere except for the $t$-th $p$-value which is set to be 0. That is, 
	\begin{equation*}
	\bar{p}_i=\begin{cases} 0 &\text{if } i=t\\
	p_i &\text{otherwise}.
	\end{cases}
	\end{equation*}
	Also let the hallucinated Laplace noises $\bar{Z}_{1:k}^t$ be an identical copy of $Z_{1:k}$, and let $\bar{Z}_\alpha$ be an identical copy of $Z_\alpha$. We denote $\bar{C}_{1:k}$ and $\bar{R}_{1:k}$ as the candidates and rejections made using $\bar{p}_{1:k}^{t}$ and $\bar{Z}_{1:k}^t$. 
  By construction, we have $\bar{R}_i=R_i$ and $\bar{C}_i=C_i$ for all $i<t$. On the event that $\{\log p_t+Z_t\le \log \alpha_t+Z_\alpha-A\}$, since $\bar{p}_t=0$ and we inject the same Laplace noise, we have $\bar{R}_t=R_t=1$ and $\bar{C}_t=C_t=1$, and hence also $\bar{R}_{1:k}=R_{1:k}$. Then the following equation holds:
	\begin{equation*}
	\frac{I(\log p_t +Z_t \le \log \alpha_t+Z_\alpha-A)}{h(R_{1:k})}=\frac{I(\log p_t +Z_t \le \log \alpha_t+Z_\alpha-A)}{h(\bar{R}_{1:k})}.
	\end{equation*}
    We note that when $\log p_t+Z_t > \log \alpha_t+Z_\alpha-A$, the above equation still holds since both sides will be zero.  Since $\bar{R}_{1:k}$ and $Z_t$, $Z_\alpha$ are independent of $p_t$, we can take conditional expectations to obtain
	\begin{align}
	\E{\frac{I(\log p_t +Z_t \le \log \alpha_t+Z_\alpha-A)}{h(R_{1:k})}|\mathcal{F'}^{t-1}}&=\E{\frac{I(\log p_t +Z_t \le \log \alpha_t+Z_\alpha-A)}{h(\bar{R}_{1:k})}|\mathcal{F'}^{t-1}}\notag\\
	&\le \E{\frac{\min\{\alpha_t\exp(Z_\alpha-Z_t-A),1\})}{h(\bar{R}_{1:k})}|\mathcal{F'}^{t-1}} \label{eq.superuni2}\\
	&\le \E{\frac{\min\{\alpha_t\exp(Z_\alpha-Z_t-A),1\})}{h(R_{1:k})}|\mathcal{F'}^{t-1}} \label{eq.barR},
	\end{align}
	where Inequality \eqref{eq.superuni2} follows by taking expectation only with respect to $p_t$ by invoking the conditional uniformity property and the fact that the support of p-values is $[0,1]$, and Inequality \eqref{eq.barR} follows from the facts that $h(R_{1:k})\le h(\bar{R}_{1:k})$ since $R_i\le \bar{R}_i$ for all $i$ and that the function $h$ is non-decreasing.
\end{proof}

\end{document}